\documentclass{article}

\usepackage{gensymb}
\usepackage[utf8]{inputenc} 
\usepackage[T1]{fontenc}    
\usepackage{hyperref}       
\usepackage{url}            
\usepackage{booktabs}       
\usepackage{amsfonts}       
\usepackage{nicefrac}       
\usepackage{microtype}      
\usepackage{wrapfig}
\usepackage{amsmath}
\usepackage{amsthm}
\usepackage{amssymb}
\usepackage{graphicx}
\usepackage[colorinlistoftodos]{todonotes}
\usepackage[ruled, linesnumbered]{algorithm2e}
\usepackage{tikz}
\usepackage{caption}
\usepackage{subcaption}
\usepackage{float}
\usepackage{natbib}
\usepackage{fourier} 
\usepackage{array}
\usepackage{makecell}
\usetikzlibrary{arrows}
\usetikzlibrary{shapes,snakes}
\usetikzlibrary{shapes.multipart}
\usetikzlibrary{external}
\tikzset{
  treenode/.style = {align=center, inner sep=0pt, text centered,
    font=\sffamily},
  arn_n/.style = {treenode, circle, white, font=\sffamily\bfseries, draw=black,
    fill=black, text width=1.5em},
  arn_r/.style = {treenode, circle, red, draw=red, 
    text width=1.5em, very thick},
  arn_x/.style = {treenode, rectangle, draw=black,
    minimum width=0.5em, minimum height=0.5em},
    arn_d/.style = {treenode, diamond, white,  draw=black,
    fill=black, text width=1.5em},
    arn_t/.style = {treenode, diamond, red, draw=red, text width=1.5em, very thick},
}
\tikzstyle{block} = [rectangle, draw, fill=blue!20, 
    text width=5em, text centered, rounded corners, minimum height=4em]
\newtheorem{theorem}{Theorem}
\newtheorem{corollary}{Corollary}[theorem]

\newtheorem{proposition}[theorem]{Proposition}
\newtheorem*{lemma*}{Lemma}	
\usepackage{hyperref}

\usepackage[accepted]{icml2019}

\icmltitlerunning{Learning to Search via Retrospective Imitation}

\begin{document}

\twocolumn[
\icmltitle{Learning to Search via Retrospective Imitation}



\icmlsetsymbol{equal}{*}

\begin{icmlauthorlist}
\icmlauthor{Jialin Song}{equal,caltech}
\icmlauthor{Ravi Lanka}{equal,jpl}
\icmlauthor{Albert Zhao}{caltech}
\icmlauthor{Aadyot Bhatnagar}{caltech}
\icmlauthor{Yisong Yue}{caltech}
\icmlauthor{Masahiro Ono}{jpl}
\end{icmlauthorlist}

\icmlaffiliation{caltech}{California Institute of Technology}
\icmlaffiliation{jpl}{Jet Propulsion Laboratory}

\icmlcorrespondingauthor{Jialin Song}{jssong@caltech.edu}

\icmlkeywords{Machine Learning, ICML}

\vskip 0.3in
]


\printAffiliationsAndNotice{\icmlEqualContribution} 

\begin{abstract}
We study the problem of learning a good search policy for combinatorial search spaces.
We propose retrospective imitation learning, which, after initial training by an expert, improves itself by learning from \textit{retrospective inspections} of its own roll-outs. That is, when the policy eventually reaches a feasible solution in a combinatorial search tree after making mistakes and backtracks, it retrospectively constructs an improved search trace to the solution by removing backtracks, which is then used to further train the policy. 
A key feature of our approach is that it can iteratively scale up, or transfer, to larger problem sizes than those solved by the initial expert demonstrations, thus dramatically expanding its applicability beyond that of conventional imitation learning. We showcase the effectiveness of our approach on a range of tasks, including synthetic maze solving and combinatorial problems expressed as integer programs. 
%
%
%
%
\end{abstract}

\section{Introduction}
\label{intro}



Many challenging tasks involve traversing a combinatorial search space. Examples include  
branch-and-bound for constrained optimization problems \citep{lawler1966branch}, A* search for path planning \citep{hart1968formal} and  game playing, e.g. Go \citep{silver2016mastering}. 
Since the search space often grows exponentially with problem size, one key challenge is how to prioritize traversing the search space. A conventional approach is to manually design heuristics that exploit specific structural assumptions (cf. \citet{gonen2000optimal,holmberg2000lagrangian}). However, this conventional approach is labor intensive and relies on human experts developing a strong understanding of the structural properties of some class of problems. 
%
%
%
%
%
%
%
%
%

 In this paper, we take a learning approach to finding an effective search heuristic. We cast the problem as policy learning for sequential decision making, where the environment is the combinatorial search problem. Viewed in this way, a seemingly natural approach to consider is reinforcement learning, where the reward comes from finding a feasible terminal state, e.g. reaching the target in A* search. However, in our problem, most terminal states are not feasible, so the reward signal is sparse; hence, we do not expect reinforcement learning approaches to be effective.


We instead build upon imitation learning \citep{ross2010efficient,ross2010reduction,daume2009search,he2014learning}, which is a promising paradigm here since an initial set of solved instances (i.e., demonstrations) can often be obtained from  existing solvers, which we also call experts.
However, obtaining solved instances can be expensive, especially for large problems. Hence, one key challenge is to avoid repeatedly querying experts during {\it training}.

We propose the \textit{retrospective imitation} approach, where the policy can iteratively learn from its own mistakes without repeated expert feedback.  Instead, we use a {\it retrospective oracle} to generate feedback by  querying the environment on rolled-out search traces (e.g., which part of the trace led to a feasible terminal state) to find the shortest path in hindsight (retrospective optimal trace).

Our approach improves upon previous imitation approaches \citep{ross2010efficient,ross2010reduction,he2014learning} in two aspects.  First, our approach iteratively refines towards solutions that may be higher quality or easier for the policy to find than the original  demonstrations.  Second and more importantly, our approach can scale to larger problem instances than the original demonstrations, allowing our approach to scale up to problem sizes beyond those that are solvable by the expert, and dramatically extending the applicability beyond that of conventional imitation learning.
We also provide a theoretical characterization for a restricted setting of the general learning problem.
%
%
%
%
%
%

We evaluate on two types of search problems: A* search and branch-and-bound in integer programs. 
We demonstrate that our approach improves upon prior imitation learning work \citep{he2014learning} as well as commercial solvers such as Gurobi (for integer programs).  We further demonstrate  generalization ability by learning to solve larger problem instances than contained in the original training data.  

In summary, our contributions are:
\vspace{-0.1in}
\begin{itemize}
\item We propose retrospective imitation, a general learning framework that generates feedback (via querying the environment) for imitation learning, without repeatedly querying experts.
\vspace{-0.05in}
\item We show how retrospective imitation can scale up beyond the problem size where demonstrations are available, which significantly expands upon the capabilities of imitation learning.
\vspace{-0.05in}
\item We provide theoretical insights on when retrospective imitation can provide improvements over imitation learning, such as when we can reliably scale up.
\vspace{-0.05in}
\item We evaluate empirically on three combinatorial search environments and show improvements over both imitation learning baselines and off-the-shelf solvers.
\end{itemize}




\section{Related Work}
\label{related}

Driven by availability of demonstration data, imitation learning is an increasingly popular learning paradigm, whereby a policy is trained to mimic the decision-making of an expert or oracle \citep{daume2009search,ross2010efficient,ross2010reduction,chang2015learning}. Existing approaches often rely on having access to a teacher at training time to derive learning signals from. In contrast, our retrospective imitation approach can learn from its own mistakes as well as train on larger problem instances than contained in the original supervised training set.

Another popular paradigm for learning for sequential decision making is reinforcement learning (RL) \citep{sutton1998reinforcement}, especially with recent success of using deep learning models as policies \citep{lillicrap2015continuous, mnih2015human}. One major challenge with RL is effective and stable learning when rewards are sparse, as in our setting. In contrast, the imitation learning reduction paradigm \cite{ross2010reduction,chang2015learning} helps alleviate this problem by reducing the learning problem to cost-sensitive classification, which essentially densifies the reward signals.

More generally, machine learning approaches have been used in other optimization settings as well, such for Boolean satisfiability \citep{boyan1998learning},  submodular optimization given side information \citep{SCP},  memory controllers \citep{ipek2008self}, device placement \cite{mirhoseini2017device}, SMT solvers \citep{balunovic2018learning}, and parameter tuning of other solvers \citep{hutter2010automated}.

Our retrospective imitation approach bears some affinity to other imitation learning approaches that aim to exceed the performance of the oracle teacher \citep{chang2015learning}.  One key difference is that we are effectively using retrospective imitation as a form of transfer learning by learning to solve problem instances of increasing size.

Another paradigm for learning to optimize is to learn a policy on-the-fly by using the first few iterations of optimization for training \cite{ipek2008self,khalil2016learning}.  This requires dense rewards as well as stationarity of optimal behavior throughout the search.  Typically, such dense rewards are surrogates of the true sparse reward. We study a complementary setting of learning off-line from demonstrations with sparse environmental rewards.


\section{Problem Setting \& Preliminaries}
\label{problem}

\textbf{Learning Search Policies for Combinatorial Search Problems.} 
Given a combinatorial search problem $P$, a policy  $\pi$ (i.e., a search algorithm) must make a sequence of decisions to traverse a combinatorial search space to find a (good) feasible solution (e.g., a setting of integer variables in an integer program that satisfies all constraints and has good objective value). We focus on combinatorial tree search, where the navigation of the search space is organized as search trees, i.e., our ``environment'' is the search space.  
Given the current ``state'' $s_t$ (the current search tree), which contains the search history so far (e.g., a partial assignment of integer variables), the policy chooses an action $a$, usually a new node to explore, to apply to the current state $s_t$ (i.e., to extend the current partial solution) and transitions to a new state $s_{t+1}$ (a new search tree). The search terminates when a complete feasible solution is found, which we also refer to as reaching a terminal state.
Figure \ref{fig:example} depicts (among other things) example roll-outs, or search traces, of such  policies.

A typical objective is to minimize search time to a terminal state. In general, the transition function is deterministic and known, but navigating a combinatorial search space to find rare terminal states is challenging. Given a training set of problem instances, we can use a learning approach to train $\pi$ to perform well on test problem instances.

\textbf{Imitation Learning.} 
We build upon the imitation learning paradigm to learn a good search policy. Previous work assumes an expert policy $\pi_{expert}$ that provides interactive feedback on the trained policy \cite{he2014learning}. 
The expert can be a human or an (expensive) solver. However, repeated queries to the expert can be prohibitively expensive. 

Our approach is based on the idea that retrospection (with query access to environment) can also generate feedback. A search trace typically has many dead ends and backtracking before finding a terminal state.
Thus, more efficient search traces (i.e., feedback) can be retrospectively extracted by removing backtracking, 
which forms the core algorithmic innovation of our approach (see Section \ref{algo}).  
Retrospective imitation also enables a form of transfer learning, where we iteratively train policies to solve larger problems for which collecting demonstrations is infeasible (e.g., due to computational costs of the original expert solver).


\begin{figure*}[t]
\vspace{-0.1in}
\centering
{\small

\begin{tikzpicture}%
  \node [label=Expert Trace] (expert) {\includegraphics[scale=0.4]{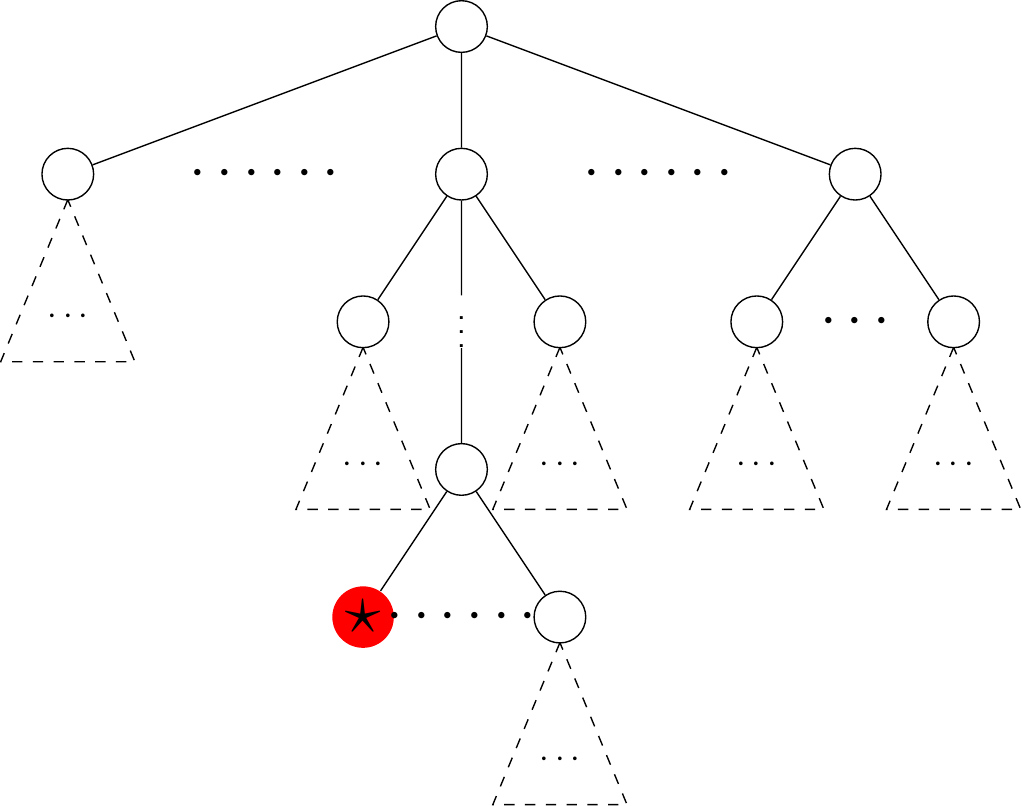}};
  \node (before_retro) [right=6cm of expert, label=Roll-out Trace] {\includegraphics[scale=0.4]{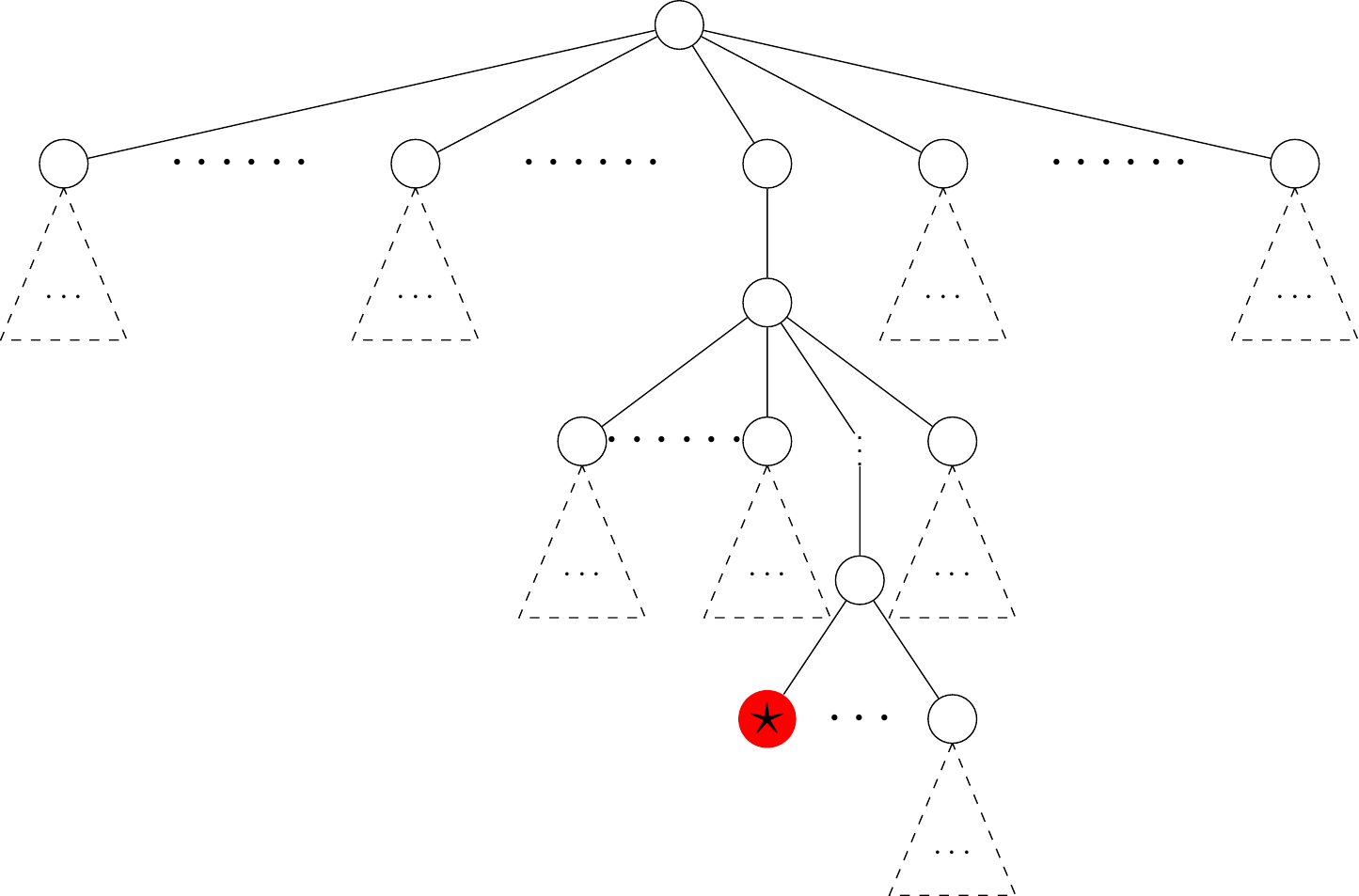}};
  \node (after_retro) [below=0.6cm of before_retro]  {\includegraphics[scale=0.4]{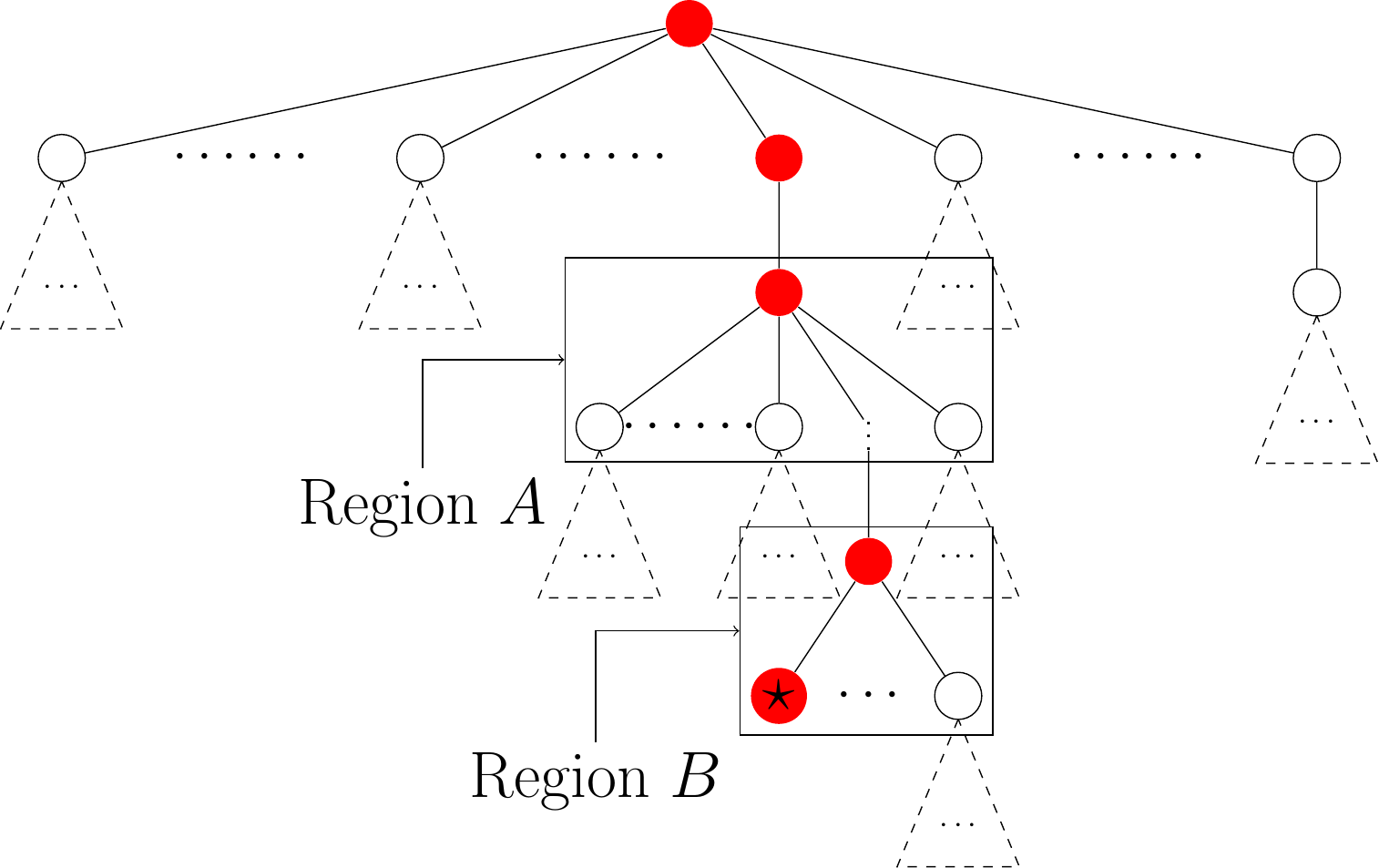}};
  \node [block, left=7.1cm of after_retro] (il_policy) {Imitation Learning Policy};
  \node (retro_feedback) [above=0.5cm] at (after_retro.south west) {Retrospective Oracle Feedback};
  \draw[->, >=stealth'] (expert) edge node[left]{\raisebox{.5pt}{\textcircled{\raisebox{-0.9pt}1}} Initial Learning} (il_policy);
  \draw[->, >=stealth'] (il_policy) edge node[above,  sloped]{\raisebox{.5pt}{\textcircled{\raisebox{-0.9pt}2}} Policy Roll-out (optional exploration)} (before_retro);
  \draw[->, >=stealth'] (before_retro) edge node[left, align=center]{\raisebox{.5pt}{\textcircled{\raisebox{-0.9pt}3}} Retrospective Oracle \\(Algorithm \ref{alg:2})} (after_retro);
  \draw[->, >=stealth'] (after_retro) edge node[above, align=center]{\raisebox{.5pt}{\textcircled{\raisebox{-0.9pt}4}} Policy Update with Further Learning} (il_policy);
\end{tikzpicture}
\vspace{-0.2in}
\caption{A visualization of retrospective imitation learning depicting components of Algorithm \ref{alg:1}. An imitation learning policy is initialized from expert traces and is rolled out to generate its own traces. Then the policy is updated according to the feedback generated by the retrospective oracle as in Figure \ref{fig:zoom}. This process is repeated until some termination condition is met.}
\label{fig:example}}
\end{figure*}
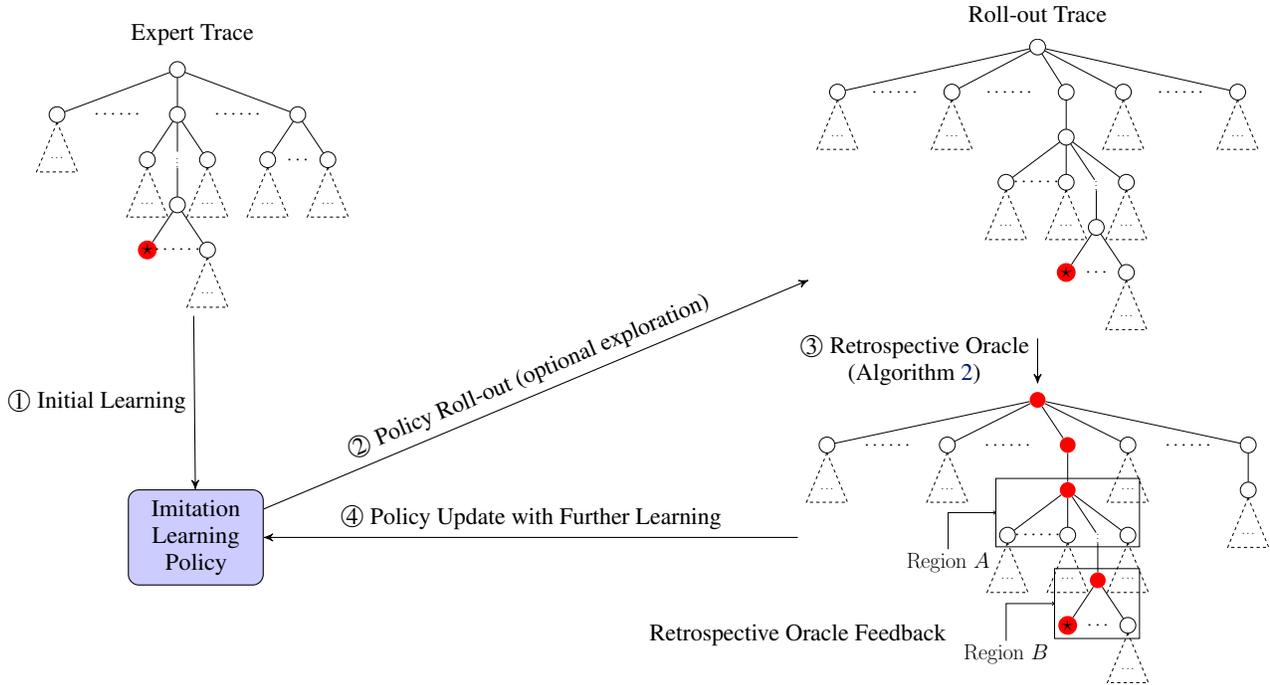

\begin{figure}
\centering
\begin{subfigure}{0.25\textwidth}
\centering
\includegraphics[scale=0.6]{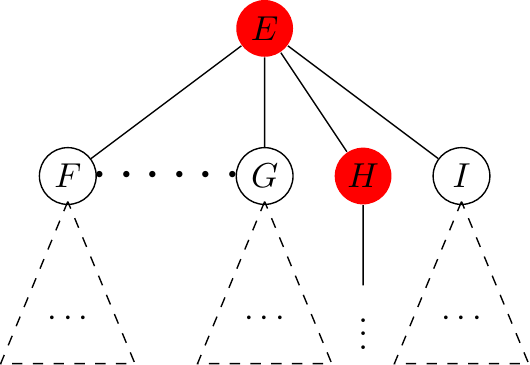}
\end{subfigure}%
~
\begin{subfigure}{0.25\textwidth}
\centering
\includegraphics[scale=0.6]{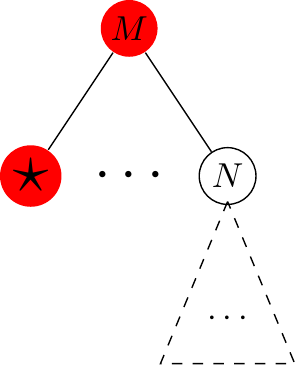}
\end{subfigure}
\vspace{-0.1in}
\caption{Zoom-in views of Region A and B in Figure \ref{fig:example}. At node $E$, the retrospective feedback indicates selecting node $H$ over $F$, $G$ and $I$. At node $M$, the $\star$ node is preferred over $N$.}
\label{fig:zoom}
\end{figure}

\section{Retrospective Imitation Learning}
\label{algo}
We now describe the retrospective imitation learning approach. It is a general framework that can be combined with a variety of imitation learning algorithms. For clarity of presentation, we instantiate our approach using the data aggregation algorithm (DAgger)  \citep{ross2010reduction,he2014learning} and we call the resulting algorithm Retrospective DAgger. We also include the instantiation with SMILe \citep{ross2010efficient} in Appendix \ref{retro_smile}. In Section \ref{sec:experiment}, we empirically evaluate retrospective imitation with both DAgger and SMILe to showcase the generality of our framework.

We decompose our general framework into two steps. First, Algorithm \ref{alg:1} describes our core procedure for learning on fixed size problems with a crucial \textit{retrospective oracle} subroutine (Algorithm \ref{alg:2}). Algorithm \ref{alg:3} then describes how to scale up beyond the fixed size. We will use Figure \ref{fig:example} as a running example. 
The ultimate goal is to enable imitation learning algorithms to scale up to problems much larger than those for which we have expert demonstrations, which is a significant improvement since conventional imitation learning cannot naturally accomplish this. 
%
%
\begin{algorithm}[t]
\begin{small}
    \SetKwInOut{Input}{Input}
    \SetKwInOut{Output}{Output}
	\LinesNumbered
    {\bf Inputs:}\\
    $N$: number of iterations \\
    $\pi_1$: initial policy trained on expert traces\\
    $\alpha$: mixing parameter\\
    $\{P_j\}$: a set of training problem instances\\
    $D_0$: expert traces dataset\\
    
    initialize $D = D_0$ \\
    \For{$i\leftarrow 1$ \KwTo $N$}
      {       
        $\hat{\pi_i} \leftarrow \alpha \pi_i + (1-\alpha)\pi_{explore}$\ \ \ (optionally  explore)\\
        run $\hat{\pi_i}$ on $\{P_j\}$ to generate a set of search traces $\{\tau_j\}$\\
        for each $\tau_j$, compute $\pi^*(\tau_j, s)$ for each terminal state $s$ (Algorithm 2)\\
        collect new dataset $D_i$ based on each $\pi^*(\tau_j, s)$\\
        update $D$ with $D_i$ (i.e., $D \leftarrow D \cup D_i$)\\
        train $\pi_{i+1}$ on $D$\\
      }
      {
        return best $\pi_i$ on validation
      }
      \end{small}
    \caption{Retrospective DAgger for Fixed Size}
    \label{alg:1}
\end{algorithm} 

\begin{algorithm}[t]
\begin{small}
    \SetKwInOut{Input}{Input}
    \SetKwInOut{Output}{Output}
    {\bf Inputs:}\\
	$\tau$: search tree trace\\
    $s \in \tau$: terminal state\\
    {\bf Output:}\\
    retro\_optimal: the retrospective optimal trace \\
    \While{$s$ is not the root}
    { 
		$\text{parent} \leftarrow s.{\text{parent}}$\\
        $\text{retro\_optimal(parent)} \leftarrow s$\\
        $s \leftarrow \text{parent}$
    }
    return retro\_optimal
    \end{small}
    \caption{Retrospective Oracle for Tree Search}
    \label{alg:2}
\end{algorithm}

\begin{algorithm}[t]
\begin{small}
    \SetKwInOut{Input}{Input}
    \SetKwInOut{Output}{Output}
    {\bf Inputs:}\\
    $S_1$: initial problem size\\
    $S_2$: target problem size\\
    $\pi_{S_1}$:  trained on expert data of problem size $S_1$\\
    \For{$s\leftarrow S_1 + 1$ \KwTo $S_2$}
    { 
      generate problem instances $\{P_i^s\}$ of size $s$\\
      train $\pi_s$ via Alg. \ref{alg:1} by running $\pi_{s-1}$ on $\{P_i^s\}$ to generate initial search traces
    }
    \end{small}
    \caption{Retrospective Imitation for Scaling Up}
    \label{alg:3}
\end{algorithm}

\textbf{Core Algorithm for Fixed Problem Size.}  We assume access to an initial dataset of expert demonstrations to help bootstrap the learning process, as described in Line 3 in Algorithm \ref{alg:1} and depicted in step \raisebox{.5pt}{\textcircled{\raisebox{-0.9pt}1}} in Figure \ref{fig:example}. Learning proceeds iteratively.  In Lines 9-10, the current policy (potentially blended with an exploration policy) runs until a termination condition, such as reaching one or more terminal states, is met. 
In Figure \ref{fig:example}, this is step \raisebox{.5pt}{\textcircled{\raisebox{-0.9pt}2}} and the red node is a terminal state. In Line 11, a retrospective oracle computes the retrospective optimal trace for each terminal state (step \raisebox{.5pt}{\textcircled{\raisebox{-0.9pt}3}}). This is identified by the path with red nodes from the root to the terminal state. In Line 12, a new dataset is generated, as discussed below. In Lines 12-14, we imitate the retrospective optimal trace (in this case using DAgger) to obtain a much more efficient search policy. We then train a new policy and repeat the process.

\textbf{Retrospective Oracle.}  A retrospective oracle (with query access to the environment) takes as input a search trace $\tau$ and outputs a retrospective optimal trace $\pi^*(\tau, s)$ for each terminal state $s$. 
Note that optimality is measured with respect to $\tau$, and not globally. That  is, based on $\tau$, what is the fastest/shortest known action sequence to reach a terminal state if we were to solve the {\it same} instance again?  In Figure \ref{fig:example}, given the current trace with a terminal state $\star$   (step \raisebox{.5pt}{\textcircled{\raisebox{-0.9pt}2}}), the retrospective optimal trace is the path along red nodes (step \raisebox{.5pt}{\textcircled{\raisebox{-0.9pt}3}}). In general, $\pi^*(\tau, s)$ will be shorter than $\tau$, which implies faster search in the original problem.
 Algorithm \ref{alg:2} shows the retrospective oracle for tree-structured search.
 Identifying a retrospective optimal trace given a terminal state is equivalent to following parent pointers until the initial state, as this results in the shortest trace. 

\textbf{Design Decisions in Training Data Creation.} Algorithm \ref{alg:1} requires specifying how to create each new dataset $D_i$ given the search traces and the retrospective optimal ones (Line 12 of Algorithm \ref{alg:1}). 
Intuitively $D_i$ should show how to correct mistakes made during roll-out to reach a terminal state $s$. What constitutes a mistake is influenced by the policy's actions. 
For reduction-based imitation learning algorithms such as DAgger and SMILe, the learning reduction converts the retrospective optimal solution into per state-action level supervised learning labels. Two concrete examples are shown in Figure \ref{fig:zoom}. 

Furthermore, in the case that  $\tau$ contains multiple terminal states, we also need to decide which to prioritize. See Section \ref{sec:experiment} for concrete instantiations of these decisions for learning search heuristics for solving mazes and learning branch-and-bound heuristics for solving integer programs.

\textbf{Scaling Up.} The most significant benefit of retrospective imitation is the ability to scale up to problems of sizes beyond those in the initial dataset of expert demonstrations.
Algorithm \ref{alg:3} describes our overall framework, which iteratively learns to solve increasingly larger instances using Algorithm \ref{alg:1} as a subroutine.  
We show in the theoretical analysis that, under certain assumptions, retrospective imitation is guaranteed able to scale, or transfer, to increasingly larger problem instances. The basic intuition is that slightly larger problem instances are often ``similar enough'' to the current size problem instances, so that the current learned policy can be used as the initial expert when scaling up.

\textbf{Incorporating Exploration.}
In practice, it can be beneficial to employ some exploration. Exploration is typically more useful when scaling up to larger problem instances. We discuss some exploration approaches in Appendix~\ref{sec:explore}.

\section{Theoretical Results}
%
%
\label{analysis}
In this section, we provide theoretical insights on when we expect retrospective imitation to improve reduction based imitation learning algorithms, such as DAgger and SMILe.

For simplicity, we regard all terminal states as equally good, so we simply aim to find one as quickly as possible.  Note that our experiments evaluate settings beyond those covered in the theoretical analysis.  All proofs are in the appendix.

Our analysis builds on a \textbf{trace inclusion assumption}: \textit{the search trace $\tau_1$ generated by a trained policy contains the trace $\tau_2$ by an expert policy.} While somewhat strict, this assumption allows us to rigorously characterize retrospective imitation when scaling up. 
We measure the quality of a policy using the following error rate: 

\vspace{-0.2in}
\begin{small}
$$\epsilon = \frac{\#\text{Non-optimal actions compared to retrospective optimal trace}}{\#\text{Actions to reach a terminal state in retrospective optimal trace}}.$$
\end{small}

\vspace{-0.2in}
Intuitively, this metric  measures how often a policy fails to agree with the retrospective oracle. 
The following proposition states that retrospective imitation can effectively scale up and obtain a lower error rate. 

\begin{proposition}
\label{prop:error}
Let $\pi_{S_1}$ be a policy trained using  imitation learning on problem size $S_1$. If, during the scaling-up training process to problems of size $S_2 > S_1$, the trained policy search trace, starting from $\pi_{S_1}$, always contains the (hypothetical) expert search trace on problem of size $S_2$ (trace inclusion assumption), then the final error rate $\epsilon_{S_2}$ is at most that obtained by running imitation learning  (with expert demonstrations) directly on problems of size $S_2$.
\end{proposition}

Next we analyze how lower error rates impact the number of actions to reach a terminal state. 
We restrict ourselves to decision spaces of size 2: branch to one of its children or backtrack to its parent. 
Theorem \ref{expectation} equates the number of actions to hitting time for an asymmetric random walk. 

\begin{theorem}
\label{expectation}
Let $\pi$ be a trained policy that has an error rate of $\epsilon \in (0, \frac{1}{2})$ as measured against the retrospective feedback. Let $P$ be a search problem where the optimal action sequence has length $N$, and let $T$ be the number of actions by $\pi$ to reach a terminal state. Then the expected number of actions by $\pi$ to reach a terminal state is $\mathbb{E}[T] =  \frac{N}{1 - 2 \epsilon}$. Moreover, $\mathbb{P}[T \ge \alpha N] \in O(\exp(-\alpha + \mathbb{E}[T] / N))$ for any $\alpha \ge 0$.
\end{theorem}
This result implies that lower error rates lead to shorter search time (in the original search problem) with exponentially high probability. By combining this result with the lower error rate of retrospective imitation (Proposition \ref{prop:error}), we see that retrospective imitation has a shorter expected search time than the corresponding imitation learning algorithm. We provide further analysis in the appendix.

\section{Experimental Results} 
\label{sec:experiment}
We empirically validate the generality of  our retrospective imitation technique by instantiating it with two well-known imitation learning algorithms, DAgger \citep{ross2010reduction} and SMILe \citep{ross2010efficient}. Appendix \ref{retro_smile} describes how to instantiate retrospective imitation with  SMILe instead of DAgger. We showcase the scaling up ability of retrospective imitation by only using demonstrations on the smallest problem size and scaling up to larger sizes in an entirely unsupervised fashion through Algorithm \ref{alg:3}. 
We experimented on both A* search and branch-and-bound search for integer programs.

\subsection{Environments and Datasets}
\label{sec:setup}

We experimented on three sets of tasks, as described below.
Branch-and-bound search for integer programs is particularly challenging and characterizes a common practical use case, so  we also include additional comparisons using datasets in \citet{he2014learning} in Appendix \ref{metric} for completeness. 

\textbf{Maze Solving with A* Search.} We generate random mazes according to the Kruskal's algorithm \citep{kruskal1956shortest}. 
For imitation learning, we use search traces provided by an A* search procedure equipped with the Manhattan distance heuristic as initial expert demonstrations. 


We experiment on mazes of 5 increasing sizes, from $11\times 11$ to $31 \times 31$. For each size, we use 48 randomly generated mazes for training, 2 for validation and 100 for testing. We perform A* search with Manhattan distance as the search heuristic to generate initial expert traces which are used to train imitation learning policies. The learning task is to learn a priority function to decide which locations to prioritize and show that it leads to more efficient maze solving. For our retrospective imitation policies, we only assume access to expert traces of maze size $11\times 11$ and learning on subsequent sizes is carried out according to Algorithm \ref{alg:3}. Running retrospective imitation resulted in generating $\sim 100$k individual data points.

\textbf{Integer Programming for Risk-aware Path Planning.}
We consider the risk-aware path planning problem from \citet{ono2013probabilistic}.
We briefly describe the setup, and defer a detailed description to Appendix~\ref{sec:MILP-formulation}.
Given a start point, a goal point, a set of polygonal obstacles, and an upper bound of the probability of failure (risk bound), the goal is to find a path, represented by a sequence of way points, that minimizes a cost while limiting the probability of collision to within the risk bound. This task can be formulated as a mixed integer linear program (MILP) \citep{Schouwenaars_MILP01,Prekopa_1999}, 
which is often solved using branch-and-bound \citep{land1960automatic}. Recently, data-driven approaches that learn branching and pruning decisions have been studied \citep{he2014learning, alvarez2014supervised, khalil2016learning}. Solving MILPs is in general NP-hard. 

We experiment on a set of $150$ instances of randomly generated obstacle maps with 10 obstacles each. We used a commercially available MILP solver Gurobi (Version $6.5.1$) to generate expert solutions. Details on dataset generation can be found in Appendix~\ref{dataset}. The risk bound was set to $\delta = 0.02$. We started from problems with 10 way points and scaled up to $14$ way points, in increments of $1$. The number of integer variables range from $400$ to $560$, which can be quite challenging to solve. For training, we assume that expert demonstrations by Gurobi are only available for the smallest problem size (10 way points, 400 binary variables). We use $50$ instances for each of training, validation and testing. Running retrospective imitation resulted in generating  $\sim 1.4$ million individual data points. 

\textbf{Integer Programming for Minimum Vertex Cover}. Minimum vertex cover (MVC) is a classical NP-hard combinatorial optimization problem, where the goal is to find the smallest subset of nodes in a given graph, such that every edge is adjacent to at least one node in this subset. This problem is quite challenging, and is difficult for commercial solvers even with large computational budgets. We generate random Erd\H{o}s-Renyi graphs \citep{erdos1960evolution} with varying number of nodes from 100 to 500. For each graph, its MVC problem is compiled into an integer linear program (ILP) and we use also the branch-and-bound search method to solve it. The number of integer variables range from 100 to 500. We use 15 labeled and 45 unlabeled graphs for training and test on 100 new graphs for each scale. Running retrospective imitation resulted in generating $\sim 350$k individual data points.

\subsection{Policy Learning}
\label{sec:policy}
For A* search, we learn a ranking model as the policy. The input features are mazes represented as a discrete-valued matrix indicating walls, passable squares, and the current location. 
We instantiate using neural networks with 2 convolutional layers with 32 $3\times 3$ filters each, $2\times 2$ max pooling, and a feed-forward layer with 64 hidden units.

For branch-and-bound search in integer programs, we considered two policy classes. The first follows \citep{he2014learning}, and consists of a node selection model (that prioritizes which node to consider next) and a pruning model (that rejects nodes from being branched on), which mirrors the structure of common branch-and-bound search heuristics. We use RankNet \citep{burges2005learning} as the selection model, instantiated using two layers with LeakyReLU \citep{maas2013rectifier} activation functions, and trained via cross entropy loss. For the pruning model, we train a 1-layer neural network classifier with higher cost on the optimal nodes compared to the negative nodes. We refer to this policy class as "select \& pruner". The other policy class only has the node selection model and is referred to as "select only".

The features used can be categorized into node-specific and tree-specific features. Node-specific features include an LP relaxation lower bound, objective value and node depth. Tree-specific features capture global aspects that include the integrality gap, number of solutions found, and global lower and upper bounds.  We normalize each feature to [-1,1] at each node, which is also known as query-based normalization \citep{Qin}.

\begin{figure*}[t]
                \centering
		\begin{subfigure}{0.3\textwidth}
        \centering
  		\includegraphics[width=\textwidth]{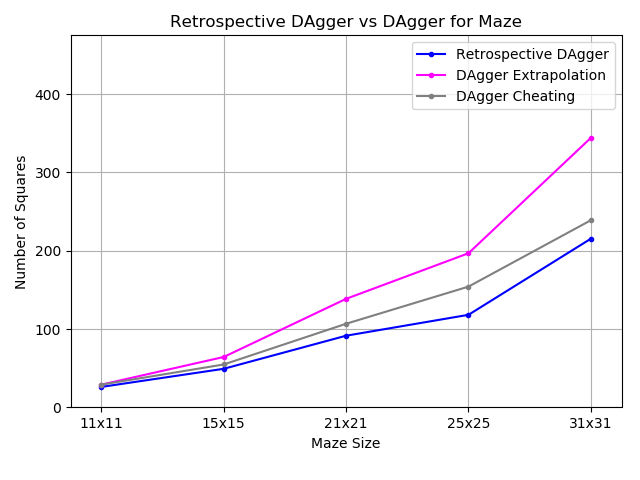}
        \vspace{-0.2in}\caption{}
  		\label{fig:dagger_squares}
        \end{subfigure} %
        ~
	    \begin{subfigure}{0.3\textwidth}
        \centering
        \includegraphics[width=\textwidth]{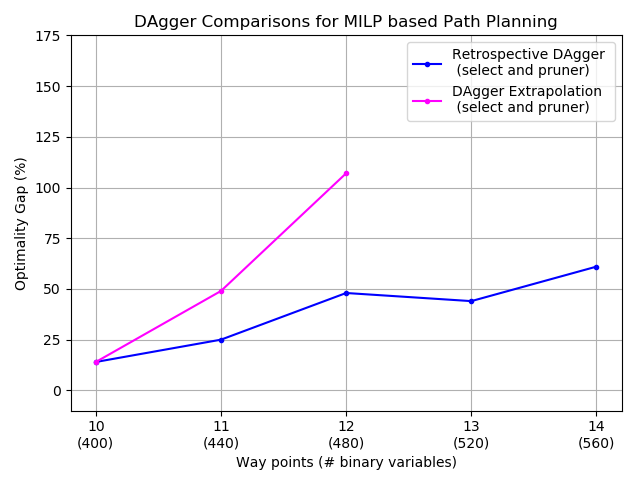}
        \vspace{-0.2in}\caption{}
        \label{fig:he_dagger_gap}
        \end{subfigure} %
        ~
		\begin{subfigure}{0.3\textwidth}
        \centering
        \includegraphics[width=\textwidth]{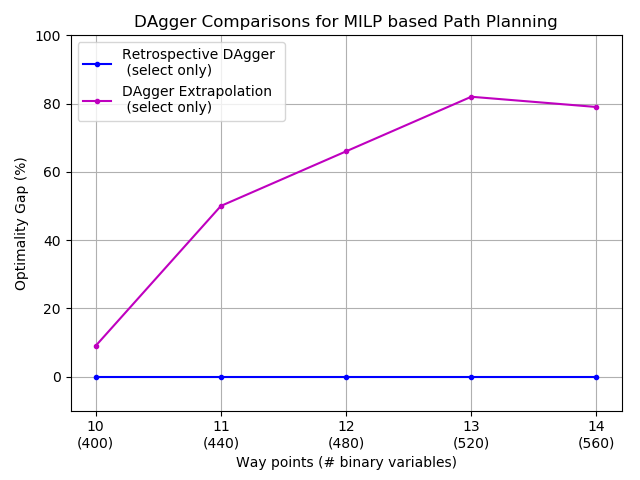}
        \vspace{-0.2in}\caption{}
        \label{fig:dagger_gap}
        \end{subfigure} %
        ~

        \begin{subfigure}{0.3\textwidth}
        \centering
  		\includegraphics[width=\textwidth]{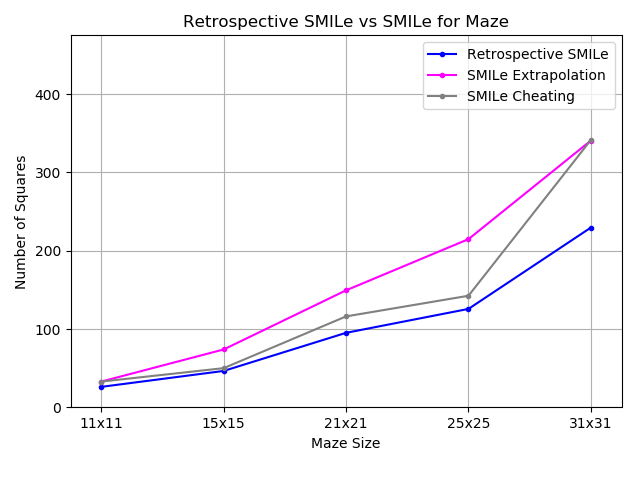}
  		\vspace{-0.2in}\caption{}
  		\label{fig:smile_squares}
        \end{subfigure} %
        ~
	    \begin{subfigure}{0.3\textwidth}
        \centering
        \includegraphics[width=\textwidth]{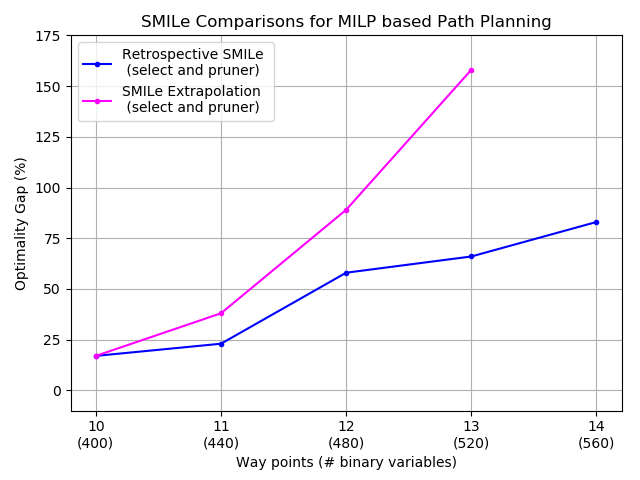}
        \vspace{-0.2in}\caption{}
        \label{fig:he_smile_gap}
        \end{subfigure} %
        ~
        \begin{subfigure}{0.3\textwidth}
        \centering
        \includegraphics[width=\textwidth]{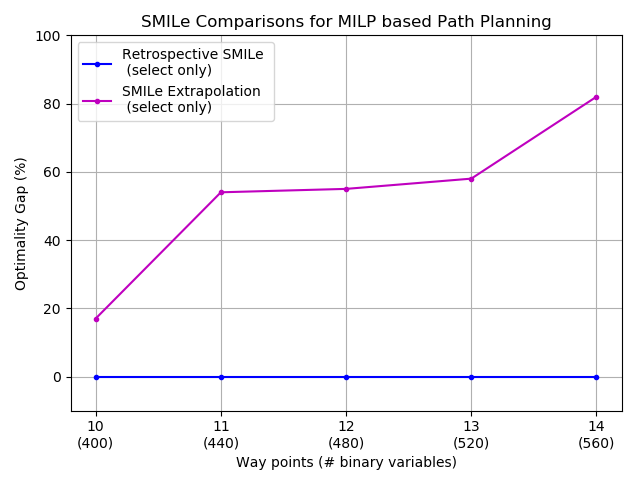}
        \vspace{-0.2in}\caption{}
        \label{fig:smile_gap}
        \end{subfigure} %
        \vspace{-0.1in}\caption{ Retrospective imitation versus DAgger (top) and SMILe (bottom) for maze solving (left) and risk-aware path planning (middle and right). ``Extrapolation'' is the conventional imitation learning baseline, and ``Cheating'' (left column only) gives imitation learning extra training data.  Retrospective imitation consistently and significantly outperforms imitation learning approaches in all settings.} 
        \label{fig:retro_vs_base}
\end{figure*}
\subsection{Main Results}
\label{sec:results}

\begin{figure}[t]
\centering
\begin{subfigure}{.14\textwidth}
  \centering
  \includegraphics[width = \textwidth]{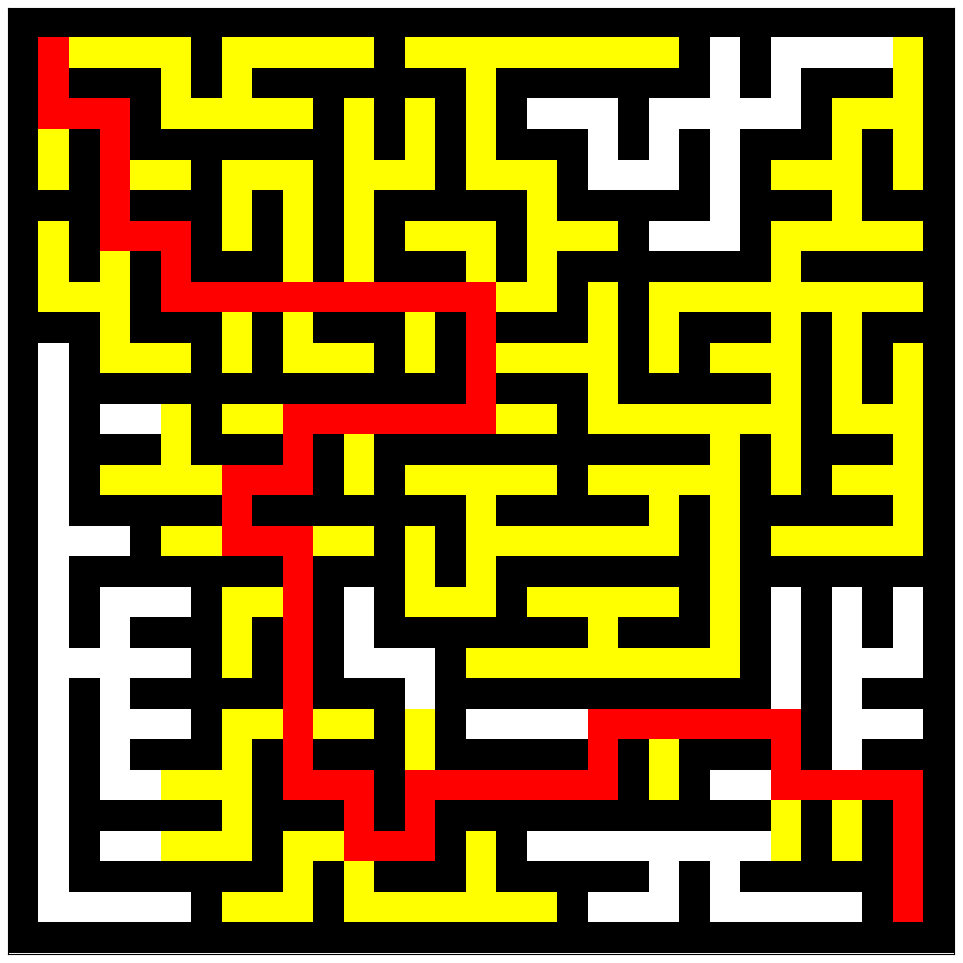}
\end{subfigure}
\  
\begin{subfigure}{.14\textwidth}
  \centering
  \includegraphics[width = \textwidth]{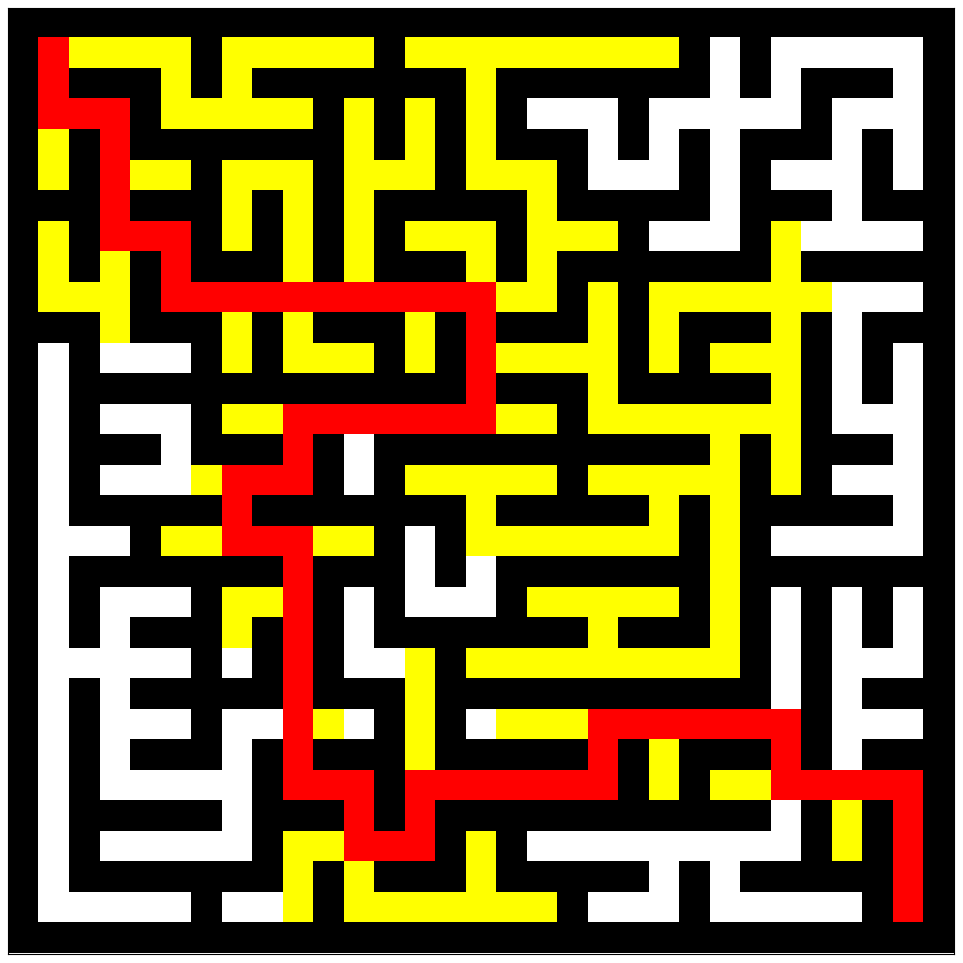}
\end{subfigure}
\ 
\begin{subfigure}{.14\textwidth}
  \centering
  \includegraphics[width = 1\textwidth]{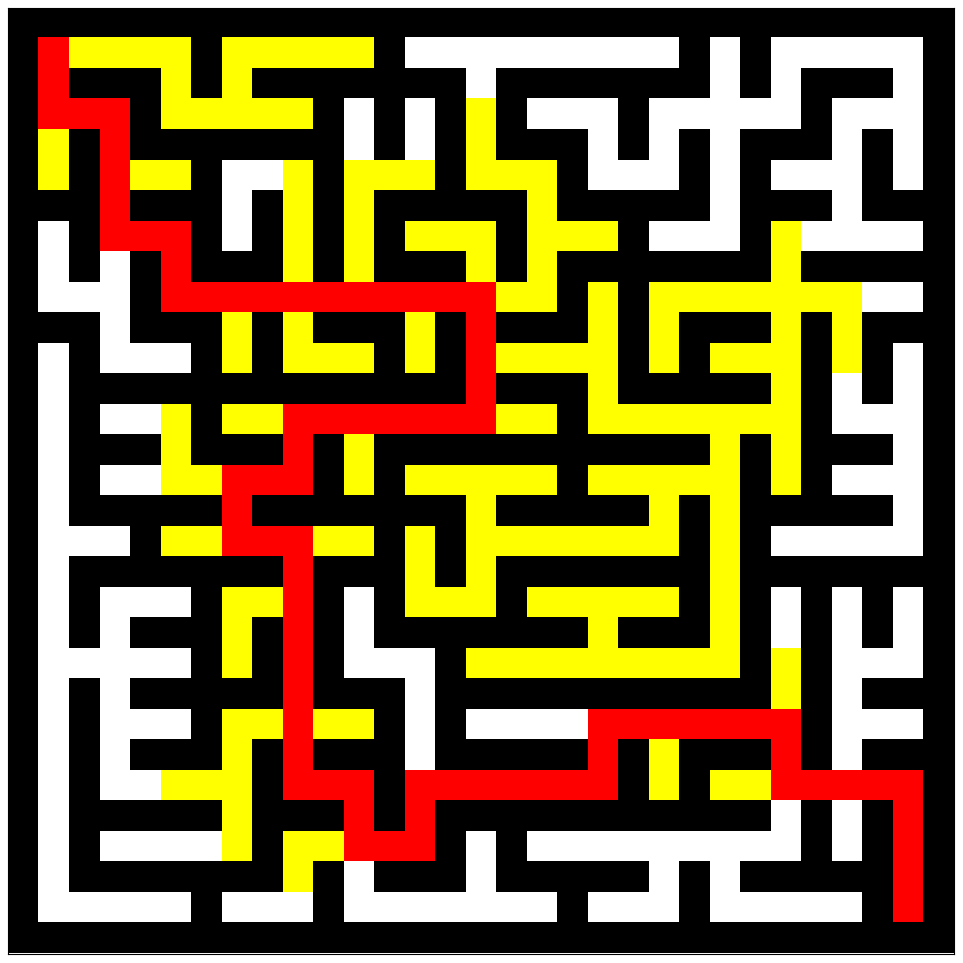}
\end{subfigure}
\vspace{-0.05in}
\caption{Left to right: comparing Manhattan distance heuristic, DAgger Cheating and Retrospective DAgger on a $31\times 31$ maze starting at upper left and ending at lower right. Yellow squares are explored. Optimal path is red. The three algorithms explore 333, 271 and 252 squares, respectively.}
\label{fig:mazes}
\end{figure}

\textbf{Comparing Retrospective Imitation with Imitation Learning.}
As retrospective learning is a general framework, we validate with two different baseline imitation learning algorithms, DAgger \citep{ross2010reduction} and SMILe \citep{ross2010efficient}. We consider two possible settings for each baseline imitation learning algorithm. The first is ``Extrapolation'', which is obtained by training an imitation model only using demonstrations on the smallest problem size and applying it directly to subsequent sizes without further learning. Extrapolation is the natural baseline to compare with retrospective imitation as both have access to the same  demonstration dataset. The second baseline setting is ``Cheating'', where we provide the baseline imitation learning algorithm with expert demonstrations on the target problem size, which is significantly more than provided to retrospective imitation.  Note that Cheating is not feasible in practice for settings of interest.  

Our main comparison results are shown in 
Figure \ref{fig:retro_vs_base}.  
For this comparison, we focus on maze solving and risk-aware path planning, as evaluating the true optimality gap for minimum vertex cover is intractable.
We see that retrospective imitation (blue) consistently and dramatically outperforms conventional Extrapolation imitation learning (magenta) in every setting. We see in Figure \ref{fig:dagger_squares}, \ref{fig:smile_squares} that retrospective imitation even outperforms Cheating imitation learning, despite having only expert demonstrations on the smallest problem size.  We also note that Retrospective DAgger consistently outperforms Retrospective SMILe. 

In the maze setting (Figure \ref{fig:dagger_squares}, \ref{fig:smile_squares}), the objective is to minimize the number of explored squares to reach the target location. 
Without further learning beyond the base size, Extrapolation degrades rapidly and the performance difference with retrospective imitation becomes very significant. Even compared with Cheating policies, retrospective imitation still achieves better objective values at every problem size, which demonstrates its transfer learning capability. Figure \ref{fig:mazes} depicts a visual comparison for an example maze. 
\begin{figure*}[t]
\centering
        \begin{subfigure}{0.3\textwidth}
        \centering
  		\includegraphics[width=\textwidth]{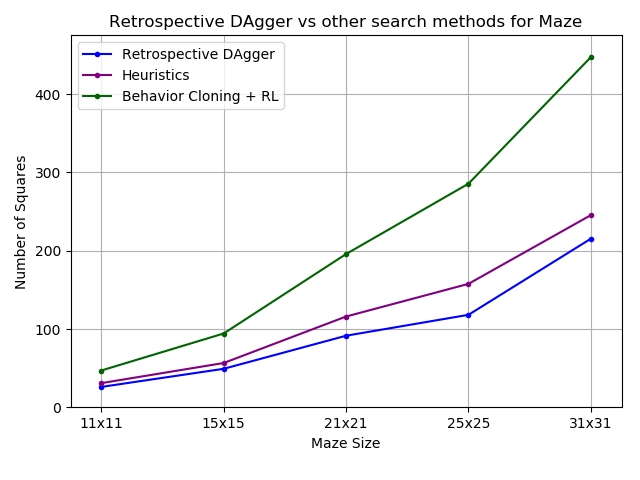}
  		\vspace{-0.25in}\caption{}
  		\label{fig:maze_benchmark}
        \end{subfigure} %
        ~
        \begin{subfigure}{0.3\textwidth}
        \centering
        \includegraphics[width=\textwidth]{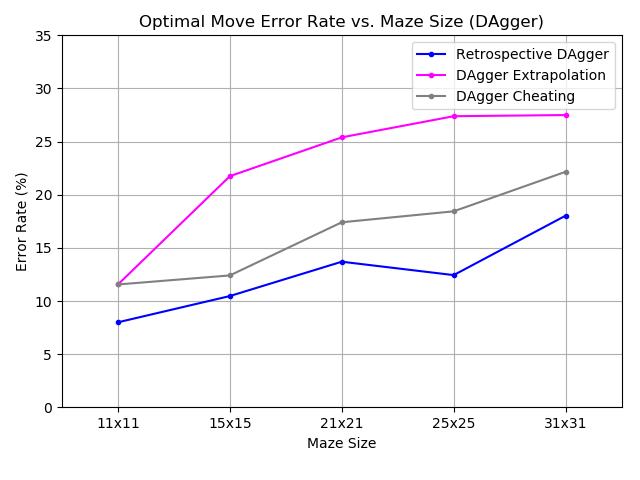}
        \vspace{-0.25in}\caption{}
        \label{fig:dagger_errors}
        \end{subfigure} %
        ~
        \begin{subfigure}{0.3\textwidth}
        \centering
        \includegraphics[width=\textwidth]{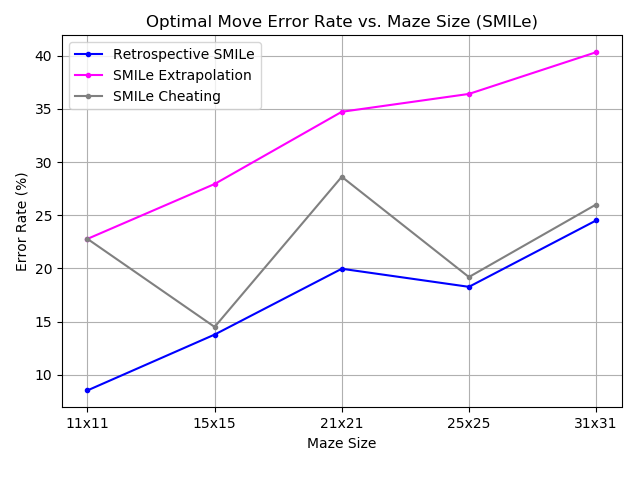}
        \vspace{-0.25in}\caption{}
        \label{fig:smile_errors}
        \end{subfigure} %
        \vspace{-.1in}\caption{(left) Retrospective imitation versus off-the-shelf methods. The RL baseline performs very poorly due to sparse environmental rewards. (middle, right) Single-step decision error rates, used for empirically validating theoretical claims.}
\end{figure*}
\begin{figure*}
\centering
	    \begin{subfigure}{0.3\textwidth}
        \centering
        \includegraphics[width=\textwidth]{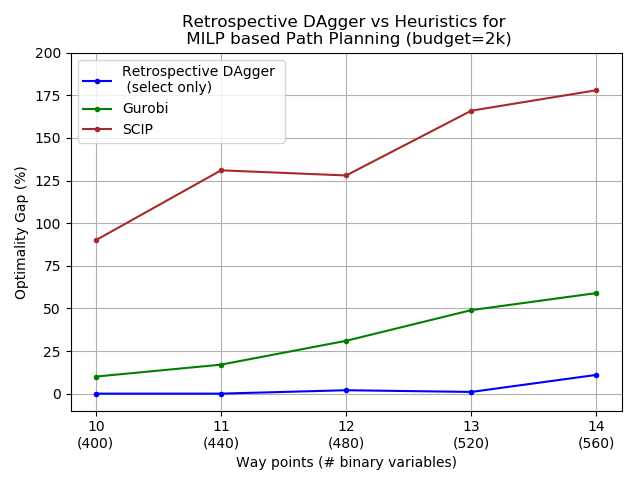}
        \vspace{-.2in}\caption{}
        \label{fig:psulu_benchmark_2k}
        \end{subfigure} %
        ~
        \begin{subfigure}{0.3\textwidth}
        \centering
        \includegraphics[width=\textwidth]{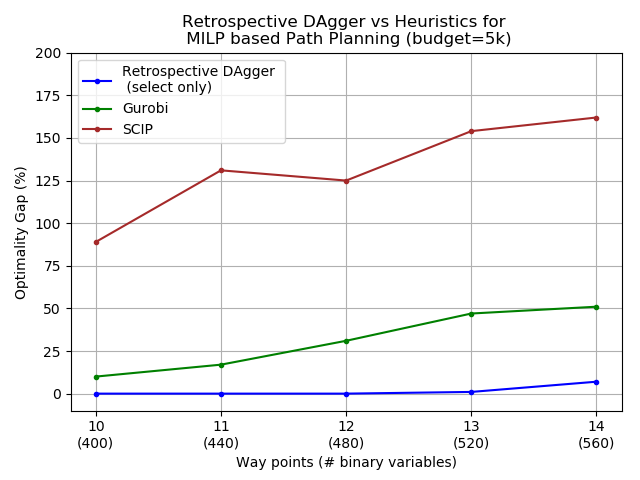}
        \vspace{-.2in}\caption{}
        \label{fig:psulu_benchmark_5k}
        \end{subfigure} %
        ~
        \begin{subfigure}{0.3\textwidth}
        \centering
        \includegraphics[width=\textwidth]{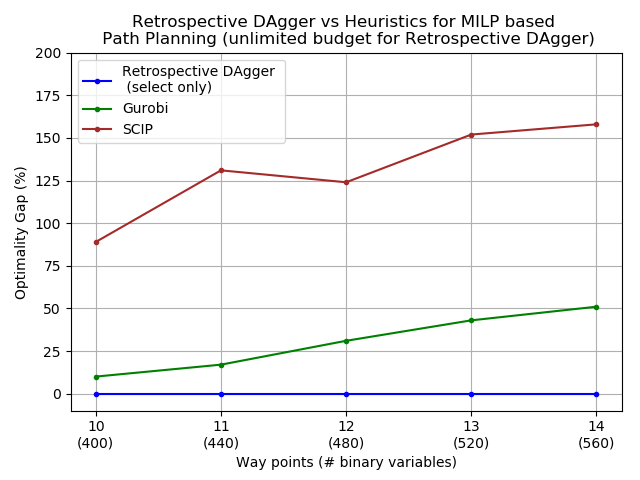}
        \vspace{-.2in}\caption{}
        \label{fig:psulu_benchmark_unlimited}
        \end{subfigure} 
        \vspace{-0.1in}
        \caption{Retrospective DAgger (``select only'' policy class) with off-the-shelf branch-and-bound solvers using various search node budgets. Retrospective DAgger consistently  outperforms baselines.}
        \label{fig:psulu}
\end{figure*}
\begin{figure}
\centering
\includegraphics[width=0.35\textwidth]{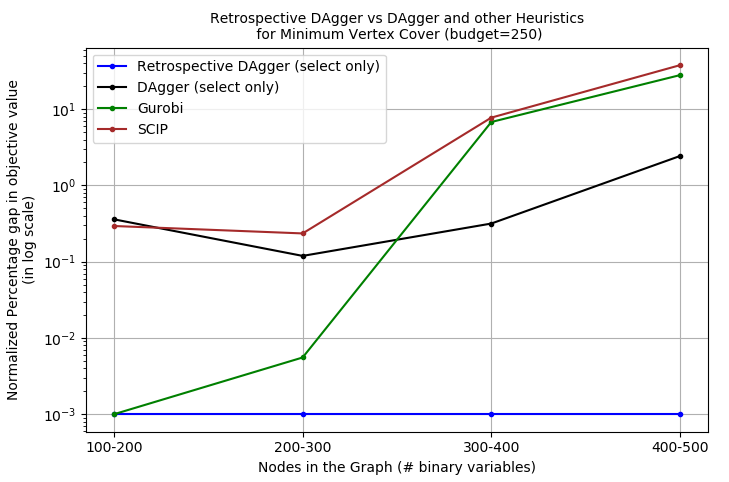}
\vspace{-.2in}
\caption{Relative objective value gaps of various methods compared with retrospective imitation when restricted with a search budget of 250 nodes. Retrospective imitation consistently outperforms other methods, especially at large scales.}
\label{fig:mvc}
\end{figure}

In the risk-aware path planning setting (Figure \ref{fig:he_dagger_gap}, \ref{fig:dagger_gap}, \ref{fig:he_smile_gap}, \ref{fig:smile_gap}), the objective is to find feasible solutions with low optimality gap, defined as the percentage difference between the best objective value found and the optimal (found via exhaustive search). If a policy fails to find a feasible solution we impose an optimality gap of 300\% to arrive at a single comparison metric. See Appendix \ref{psulu_extra} for statistics on how many problems are not solved. We compare the optimality gap of the algorithms at the same number of explored nodes. In Figure \ref{fig:he_dagger_gap}, \ref{fig:he_smile_gap} we first run the retrospective imitation version until termination, and then run the other algorithms to the same number of explored nodes. In Figure \ref{fig:dagger_gap}, \ref{fig:smile_gap}, we first run the retrospective imitation with the ``select only'' policy class until termination, and then run the other algorithms to the same number of explored nodes.  We note that the ``select only'' policy class (Figure \ref{fig:dagger_gap}, \ref{fig:smile_gap}) significantly outperforms the ``select and pruner'' policy class (Figure \ref{fig:he_dagger_gap}, \ref{fig:he_smile_gap}), which suggests that utilizing conceptually simpler policy classes may be more amenable to learning-based approaches in combinatorial search problems.

While scaling up, retrospective imitation obtains consistently low optimality gaps. In contrast, DAgger Extrapolation in Figure \ref{fig:he_dagger_gap} failed to find feasible solutions for $\sim 60\%$ test instances beyond 12 way points, so we did not test it beyond 12 way points. SMILe Extrapolation in Figure \ref{fig:he_smile_gap} failed for $\sim 75\%$ of the test instance beyond 13 way points. The fact that retrospective imitation continues to solve larger MILPs with a very slow optimality gap growth suggests that our approach is performing effective transfer learning. 

Minimum vertex cover is a challenging setting where it is infeasible to compute the optimal solution (even with large computational budgets).  We thus plot relative differences in objective with respect to retrospective imitation.  We see in Figure \ref{fig:mvc} that retrospective imitation consistently outperforms conventional imitation learning.


\textbf{Comparing Retrospective Imitation with Off-the-Shelf Approaches.} For maze solving, we compare with: 1) A* search with the Manhattan distance heuristic, and 2) behavioral cloning followed by reinforcement learning with a deep Q-network \citep{mnih2015human}. Figure \ref{fig:maze_benchmark} shows Retrospective DAgger outperforming both methods.  Due to the sparsity of the environmental rewards (only positive reward at terminal state), reinforcement learning performs significantly worse than even the Manhattan heuristic.

For risk-aware path planning and minimum vertex cover, we compare with a commercial solver Gurobi (Version 6.5.1) and SCIP (Version 4.0.1, using Gurobi as the LP solver). 
We implement our approach within the SCIP \citep{achterberg2009scip} integer programming framework.
Due to differences in  implementation, we use the number of explored nodes as a proxy for runtime. We control the search size for Retrospective DAgger (``select only'') and use its resulting search sizes to control Gurobi and SCIP.
Figures \ref{fig:psulu} \&  \ref{fig:mvc} show the results on a range of search size limits.  We see that Retrospective DAgger (``select only'') is able to consistently achieve the lowest optimality gaps, and the optimality gap grows very slowly as the number of integer variables scale far beyond the base problem scale.  As a point of comparison, the next closest solver, Gurobi, has an optimality gap $\sim 50 \%$ higher than Retrospective DAgger (``select only'') at 14 waypoints (560 binary variables) in the risk-aware path planning task and a performance gap of $\sim 40\%$ compared with Retrospective DAgger at the largest graph scale for minimum vertex cover.


\textbf{Empirically Validating Theoretical Results.}
Finally, we evaluate how well our theoretical results in Section \ref{analysis} characterizes our experimental results. Figure \ref{fig:dagger_errors} and \ref{fig:smile_errors} presents the optimal move error rates for the maze experiment, which validates Proposition \ref{prop:error} that retrospective imitation is guaranteed to result in a policy that has lower error rates than imitation learning. The benefit of having a lower error rate is explained by Theorem \ref{expectation}, which informally states that a lower error rate leads to shorter search time. This result is also verified by Figure \ref{fig:dagger_squares} and \ref{fig:smile_squares}, where Retrospective DAgger/SMILe, having the lowest error rates, explores the smallest number of squares at each problem scale.

\section{Conclusion \& Future Work}
\label{conclusion}
We have presented the retrospective imitation approach for learning combinatorial search policies. Our approach extends conventional imitation learning, by being able to learn good policies without requiring repeated queries to an expert.  
A key distinguishing feature of our approach is the ability to scale to larger problem instances than contained in the original supervised training set of demonstrations.
Our theoretical analysis shows that, under certain assumptions, the retrospective imitation learning scheme is provably more powerful and general than conventional imitation learning. 
We validated our theoretical results on a maze solving experiment and tested our approach on the problem of risk-aware path planning and minimum vertex cover, where we demonstrated both performance gains over conventional imitation learning and the ability to scale up to large problem instances not tractably solvable by commercial solvers.

By removing the need for repeated expert feedback, retrospective imitation offers the potential for increased applicability over imitation learning in search settings. However, human feedback is still a valuable asset as
human computation has been shown to boost performance of certain hard search problems \citep{le2014human}. It will be interesting to incorporate human computation into the retrospective imitation learning framework so that we can find a balance between manually instructing and autonomously reasoning to learn better search policies. 
Retrospective imitation lies in a point in the spectrum between imitation learning and reinforcement learning; we are interested in exploring other novel learning frameworks in this spectrum as well.


\newpage

\bibliography{references}
\bibliographystyle{icml2019}

\clearpage
\appendix

\newcommand{\E}[1]{
	{\mathbb{E}}\left[{#1}\right]
}

\renewcommand{\P}[1]{
	{\mathbb{P}}\left[{#1}\right]
}

\section*{Supplementary Material}
\appendix

\section{Retrospective Imitation with SMILe}
\label{retro_smile}
\begin{algorithm}[h]
\begin{small}
    \SetKwInOut{Input}{Input}
    \SetKwInOut{Output}{Output}
	\LinesNumbered
    {\bf Inputs:}\\
    $N$: number of iterations \\
    $\pi_1$: initially trained on expert traces\\
    $\alpha$: mixing parameter\\
    $\{P_j\}$: a set of problem instances\\

    \For{$i\leftarrow 1$ \KwTo $N$}
      {       
        run $\pi_i$ on $\{P_j\}$ to generate trace $\{\tau_j\}$\\
        compute $\pi^*(\tau_j, s)$ for each terminal state $s$ (Algorithm 2)\\
        collect new dataset $D$ based on $\pi^*(\tau_j, s)$\\
        train $\hat{\pi}_{i+1}$ on $D$\\
        $\pi_{i+1} = (1-\alpha)^i \pi_1 + \alpha\sum\limits_{j=1}^i (1-\alpha)^{j-1}\hat{\pi}_j$
      }
      {
        return $\pi_{N+1}$
      }
      \end{small}
    \caption{Retrospective SMILe}
    \label{alg:4}
\end{algorithm} 

\section{Additional Theoretical Results and Proofs}
\label{proofs}

First we prove Proposition \ref{prop:error}.
\begin{proof}
By the trace inclusion assumption, the dataset obtained by retrospective imitation will contain feedback for every node in the expert trace. Furthermore, the retrospective oracle feedback corresponds to the right training objective while the dataset collected by imitation learning does not, as explained in Section \ref{algo}. So the error rate trained on retrospective imitation learning data will be at most that of imitation learning.
\end{proof}
Our next theoretical result demonstrates that if larger problem instances have similar optimal solutions, a policy will not suffer a large increase in its error rate, i.e., we can ``transfer'' to larger sizes effectively. We consider the case where the problem size increase corresponds to a larger search space, i.e., the underlying problem formulation stays the same but an algorithm needs to search through a larger space. Intuitively, the following result shows that a solution from a smaller search space could already satisfy the quality constraint. Thus, a policy trained on a smaller scale can still produce satisfactory solutions to larger scale problems.
\begin{proposition}
\label{transition}
For a problem instance $P$, let $v_k^*$ denote the best objective value for $P$ when the search space has size $k$. Assume an algorithm returns a solution with objective value $v_k$, with  $v_k \ge \alpha v_k^*$ with $\alpha\in (0, 1)$. Then for any $\beta > 0$, there exists $K$ such that $v_{K} \ge \alpha v_{K+1}^* - \beta$.
\end{proposition}
\begin{proof}
Since $P$ has a finite optimal objective value $v^*$, and for any $k < k'$, $v_k^* \le v_{k'}^*$, then it follows that there exists an index $K$ such that $v_{K+1}^* - v_K^* \le \frac{\beta}{\alpha}$.\\
Then it follows that $v_K \ge \alpha v_K^*\ge \alpha(v_{K+1}^* - \frac{\beta}{\alpha}) = \alpha v_{K+1}^* - \beta$.
\end{proof}
Since the slack variable $\beta$ can be made arbitrarily small, Proposition \ref{transition} implies that solutions meeting the termination condition need not look very different when transitioning from a smaller search space to a larger one.  Finally, our next corollary justifies applying a learned policy to search through a larger search space while preserving performance quality, implying the ability to scale up retrospective imitation on larger problems so long as the earlier propositions are satisfied.
\begin{corollary}
Let $\epsilon_k$ be the error rate of an algorithm searching through a search space of size $k$. Then there exists $K$ such that $\epsilon_K = \epsilon_{K+1}$.
\end{corollary}

To prove Theorem \ref{expectation} we need the following lemma on asymmetric 1-dimensional random walks.
\begin{lemma*}
Let $Z_i, i = 1, 2, \cdots$ be i.i.d.\ Bernoulli random variables with the distribution 
\[Z_i = \begin{cases}
1, \text{ with probability $1-\epsilon$}\\
-1, \text{ with probability $\epsilon$}
\end{cases}\]
for some $\epsilon\in [0, \frac{1}{2})$. Define $X_n = \sum\limits_{i=1}^n Z_i$ and the hitting time $T_N = \inf\{n: X_n = N\}$ for some fixed integer $N \ge 0$. Then $\E{T_N} = \frac{N}{1 - 2 \epsilon}$ and $\P{T_N \ge \alpha N} \in O(\exp(-\alpha + \E{T_N} / N))$.
\end{lemma*}
\begin{proof}
The proof will proceed as follows: we will begin by computing the moment-generating function (MGF) for $T_1$ and then use it to compute the MGF of $T_N$. Then, we will use this MGF to produce a Chernoff-style bound on $\P{T_N > \alpha N}$.

The key observation for computing the MGF of $T_N$ is that $T_N \overset{\mathrm{dist}}{=} \sum_{i = 1}^{N} T_1^{(i)}$, where $T_1^{(1)}, \ldots, T_1^{(N)} \overset{\mathrm{iid}}{\sim} \P{T_1}$. This is because the random walk moves by at most one position at any given step, independent of its overall history. Therefore the time it takes the walk move $N$ steps to the right is exactly the time it takes for the walk to move 1 step to the right $N$ times. So we have $\E{e^{\beta T_N}} = \E{e^{\beta T_1}}^N$ by independence.

With this in mind, let $\Phi(\lambda) = \E{\lambda^{T_1}}$ be the generating function of $T_1$. Then, by the law of total expectation and the facts above,
\begin{align*}
    \Phi(\lambda) &= \P{Z_1 = 1} \E{\lambda^{T_1} \mid Z_1 = 1} \\
    & \quad + \P{Z_1 = -1} \E{\lambda^{T_1} \mid Z_1 = -1} \\
    &= (1 - \epsilon) \E{\lambda^{1 + T_0}} + \epsilon \E{\lambda^{1 + T_2}} \\
    &= \lambda \left( (1 - \epsilon) \E{\lambda^{1 + T_0}} + \epsilon \E{\lambda^{1 + T_1^{(1)} + T_1^{(2)}}} \right) \\
    &= \lambda \left( (1 - \epsilon) + \epsilon \Phi(\lambda)^2 \right)
\end{align*}
Solving this quadratic equation in $\Phi(\lambda)$ and taking the solution that gives $\Phi(0) = \E{0^{T_1}} = 0$, we get
\begin{align*}
    \Phi(\lambda) &= \frac{1}{2 \epsilon^2 \lambda} \left(1 - \sqrt{1 - 4 \epsilon (1 - \epsilon) \lambda^2} \right)
\end{align*}
Now, we note that the MGF of $T_1$ is just $\Phi(e^\beta) = \E{e^{\beta T_1}}$. So $\E{e^{\beta T_N}} = \Phi(e^{\beta})^N$. To prove the first claim, we can just differentiate $\E{e^{\beta T_N}}$ in $\beta$ and evaluate it at $\beta = 0$, which tells us that $\E{T_N} = \frac{N}{1 - 2\epsilon}$.

To prove the second claim, we can apply Markov's inequality to conclude that for any $\alpha, \beta \ge 0$,
\begin{align*}
\P{T_N \ge \alpha N} &= \P{e^{\beta T_N} \ge e^{\beta \alpha N}} \\
&\le \E{e^{\beta T_N} e^{-\beta \alpha N}} \\
&= \left( \E{e^{\beta T_1}} e^{-\beta \alpha} \right)^N
\end{align*}
Letting $\beta = \frac{1}{N}$ and taking the limit as $N \to \infty$, we get that
\begin{align*}
    \lim\limits_{N \to \infty} \P{T_N \ge \alpha N} \le \exp\left(-\alpha + \frac{1}{1 - 2 \epsilon} \right)
\end{align*}
which implies the concentration bound asymptotically.
\end{proof}

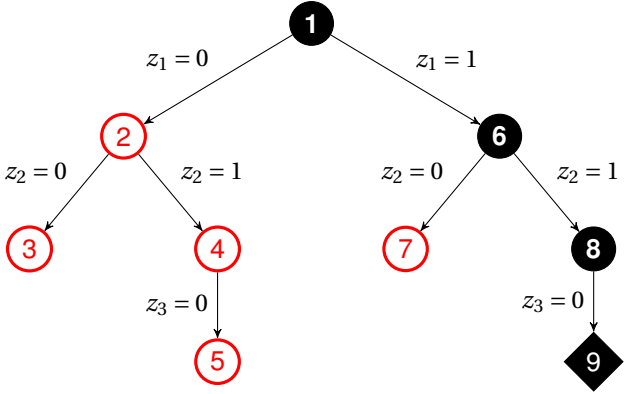
\begin{figure}[H]
\centering
\begin{tikzpicture}[->,>=stealth',level/.style={sibling distance = 5cm/#1,
  level distance = 1.5cm}] 
\node [arn_n] {1}
    child{ node [arn_r] {2} 
            child{ node [arn_r] {3} edge from parent node [above left] {$z_2 = 0$}}
            child{ node [arn_r] {4} 
            	child {node [arn_r] {5} edge from parent node [left] {$z_3 = 0$}}
            	edge from parent node [above right] {$z_2 = 1$}}  
            edge from parent node [above left] {$z_1 = 0$}
    }
    child{ node [arn_n] {6}
            child{ node [arn_r] {7} 
                            edge from parent node [above left] {$z_2 = 0$}
            }
            child{ node [arn_n] {8}
							child{ node [arn_d] {9}
                            edge from parent node [left] {$z_3 = 0$}
                            }
            edge from parent node [above right] {$z_2 = 1$}
            }
            edge from parent node [above right] {$z_1 = 1$}
		}
; 
\end{tikzpicture}
\caption{An example search trace by a policy. The solid black nodes ($1\rightarrow 6\rightarrow 8\rightarrow 9$) make up the best trace to a terminal state in retrospect. The empty red nodes are the mistakes made during this search procedure. Every mistake increases the distance to the target node (node 9) by 1 unit, while every correct decision decreases the distance by 1 unit.}
\label{search_example}
\end{figure}
Now onto the proof for the Theorem \ref{expectation}. 
\begin{proof}
We consider the search problem as a 1-dimensional random walk (see Figure \ref{search_example}). The random walk starts at the origin and proceeds in an episodic manner. The goal is to reach the point $N$ and at each time step, a wrong decision is equivalent to moving 1 unit to the left whereas a right decision is equivalent to moving 1 unit to the right. The error rate of the policy determines the probabilities of moving left and right. Thus the search problem can be reduced to 1-dimensional random walk, so we can invoke the previous lemma and assert (1) that the expected number of time steps before reaching a feasible solution is $\frac{N}{1-2\epsilon}$, and (2) that the probability that this number of time steps is greater than $\alpha N$ is $O\left( \exp\left( -\alpha + \frac{1}{1 - 2 \epsilon} \right) \right)$.
\end{proof}

This theorem allows us to measure the impact of error rates on the expected number of actions.

\begin{corollary}
With two policies $\pi_1$ and $\pi_2$ with corresponding error rates $0 < \epsilon_1 < \epsilon_2 < \frac{1}{2}$, $\pi_2$ takes $\frac{1-2\epsilon_1}{1-2\epsilon_2}$ times more actions to reach a feasible state in expectation. Moreover, the probability that $\pi_1$ terminates in $\alpha N$ time steps (for any $\alpha \ge 0$) is $\exp\left( \frac{1}{1 - 2\epsilon_2} - \frac{1}{1 - 2 \epsilon_1} \right)$ times higher.
\end{corollary}

\section{MILP formulation of risk-aware path planning}\label{sec:MILP-formulation}
This section describes the MILP formulation of risk-aware path planning solved in Section \ref{sec:experiment}. 
Our formulation is based on the MILP-based path planning originally presented by \citep{Schouwenaars_MILP01}, combined with risk-bounded constrained tightening \citep{Prekopa_1999}. 
It is a similar formulation as that of the state-of-the-art risk-aware path planner pSulu \citep{ono2013probabilistic} but without risk allocation. 

We consider a path planning problem in $\mathbb{R^n}$, where a path is represented as a sequence of $N$ way points $x_1, \cdots x_N \in X$. 
The vehicle is governed by a linear dynamics given by:
\begin{align*}
x_{k+1} & = Ax_k + Bu_k + w_k \\
u_k & \in U,
\end{align*}
where $U \subset \mathbb{R}^m$ is a control space, $u_k \in U$ is a control input, $w_k \in \mathbb{R}^n$ is a zero-mean Gaussian-distributed disturbance, and $A$ and $B$ are $n$-by-$n$ and $n$-by-$m$ matices, respectively. Note that the dynamic of the mean and covariance of $x_i$, denoted by $\bar{x}_i$ and $\Sigma_i$, respectively, have a deterministic dynamics:
\begin{align}
\bar{x}_{k+1} &= A\bar{x}_k + Bu_k + w_k  \label{eq:const1} \\
\Sigma_{k+1} &= A \Sigma A^T + W, \nonumber
\end{align}
where $W$ is the covariance of $w_k$. We assume there are $M$ polygonal obstacles in the state space, hence the following linear constraints must be satisfied in order to be safe (as in Figure \ref{fig:obtacle}):
\[
\bigwedge_{k=1}^N \bigwedge_{i=1}^M \bigvee_{j=1}^{L_i} h_{ij} x_k \le g_{ij},
\]
where $\bigwedge$ is conjunction (i.e., AND), $\bigvee$ is disjunction (i.e., OR), $L_i$ is the number of edges of the $i$-th obstacle, and $h_{ij}$ and $g_{ij}$ are constant vector and scaler, respectively. In order for each of the linear constraints to be  satisfied with the probability of $1-\delta_{kij}$, the following has to be satisfied:
\begin{align}
& \bigwedge_{k=1}^N \bigwedge_{i=1}^M \bigvee_{j=1}^{L_i} h_{ij} \bar{x}_k \le g_{ij} - \Phi(\delta_{kij}) \label{eq:const2} \\
& \Phi(\delta_{kij}) = -\sqrt{2 h_{ijk} \Sigma_{x,k} h_{ijk}^T} \ {\rm erf}^{-1}(2\delta_{ijk}-1), \nonumber
\end{align}
where ${\rm erf}^{-1}$ is the inverse error function. 


\begin{figure}
    \begin{minipage}[b]{.45\textwidth}
  		\hbox{\hspace{2ex}\includegraphics[width=0.8\textwidth]{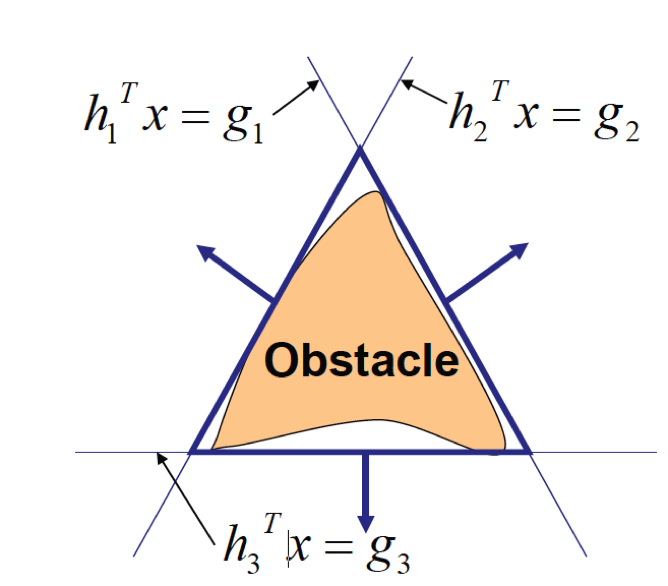}}
  		\caption{Representation of polygonal obstacle by disjunctive linear constraints \\ \\}
  		\label{fig:obtacle}
	\end{minipage}\hspace{0.05\textwidth}
    \begin{minipage}[b]{.45\textwidth}
        \hbox{\hspace{3ex}\includegraphics[width=0.8\textwidth]{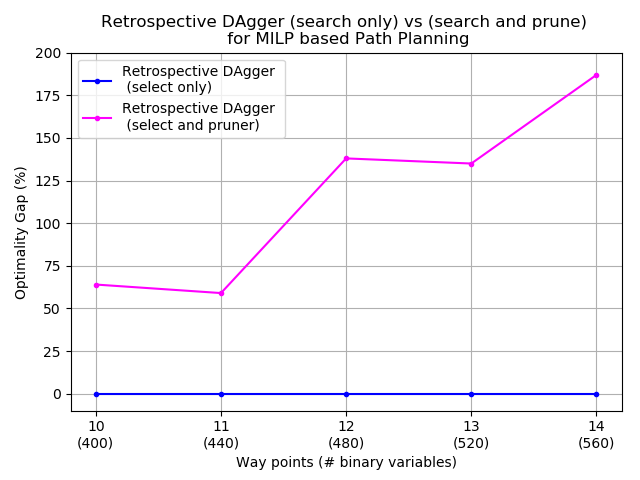}}
        \caption{Comparison of optimality gap between Retrospective DAgger (select only) and Retrospective DAgger (select and prune)}
        \label{fig:searchvsprune}
    \end{minipage}%
    \vspace{0.1in}
\end{figure}

\begin{figure}
\begin{minipage}[b]{.45\textwidth}
        \hbox{\hspace{3ex}\includegraphics[width=0.8\textwidth]{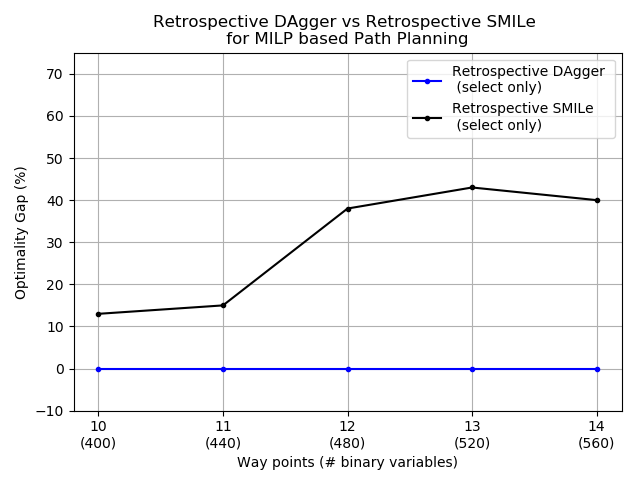}}
        \caption{Comparison of optimality gap between Retrospective DAgger (select only) and Retrospective SMILe (select only)}
        \label{fig:dagger_smile_pgap}
    \end{minipage}%
    \vspace{0.1in}
\end{figure}

The problem that we solve is, given the initial state $(\bar{x}_0, \Sigma_0)$, to find $u_1 \cdots u_N \in U$ that minimizes a linear objective function and satisfies (\ref{eq:const1}) and (\ref{eq:const2}). An arbitrary nonlinear objective function can be approximated by a piecewise linear function by introducing integer variables. The disjunction in (\ref{eq:const2}) is also replaced by integer variables using the standard Big M method. Therefore, this problem is equivalent to MILP. In the branch-and-bound algorithm, the choice of which linear constraint to be satisfied among the disjunctive constraints in (\ref{eq:const2}) (i.e., which side of the obstacle $x_k$ is) corresponds to which branch to choose at each node.

\section{Risk-aware Planning Dataset Generation}
\label{dataset}
We generate 150 obstacle maps. Each map contains $10$ rectangle obstacles, with the center of each obstacle chosen from a uniform random distribution over the space $0\leq y \leq 1$ , $0\leq x \leq 1$. The side length of each obstacle was chosen from a uniform distribution in range $[0.01, 0.02]$ and the orientation was chosen from a uniform distribution between $0\degree$ and $360\degree$. In order to avoid trivial infeasible maps, any obstacles centered close to the destination are removed. 

\section{Retrospective DAgger vs Retrospective SMILe for Maze Solving}
\label{maze_extra}
\begin{figure}[H]
\centering
\includegraphics[width=0.4\textwidth]{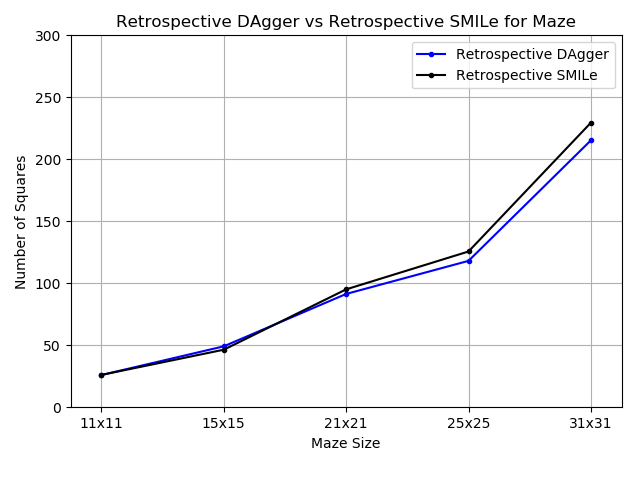}
\caption{Average explored number of squares for Retrospective DAgger and Retrospective SMILe.}
\end{figure}

\section{Additional Experiments on Risk-aware Planning}
\label{psulu_extra}
In this section, we present a comparison of Retrospective DAgger with two different policy classes for MILP based Path Planning, namely a combination of both select and prune policy as described in \citep{he2014learning} against select policy alone. We compare their optimality gap by first running the Retrospective DAgger (select only) until termination and then limiting the Retrospective Dagger (search and prune) to the same number of explored nodes. Figure \ref{fig:searchvsprune} depicts a comparison of optimality gap with varying number for waypoints. We observe that Retrospective DAgger (select only) performs much better in comparison to Retrospective DAgger (select and prune). 

Next, we present a comparison of Retrospective DAgger (select only) with Retrospective SMILe (select only). We compare the optimality gap by limiting Retrospective SMILe (select only) to the same number of nodes explored by Retrospective DAgger (select only), which is run without any node limits until termination. The results of this experiment are shown in Figure \ref{fig:dagger_smile_pgap}. Retrospective DAgger (select only) performs superior to Retrospective SMILe (select only) validating our theoretical understanding of the two algorithms.

\begin{figure}[H]
\centering
\includegraphics[width=0.5\textwidth]{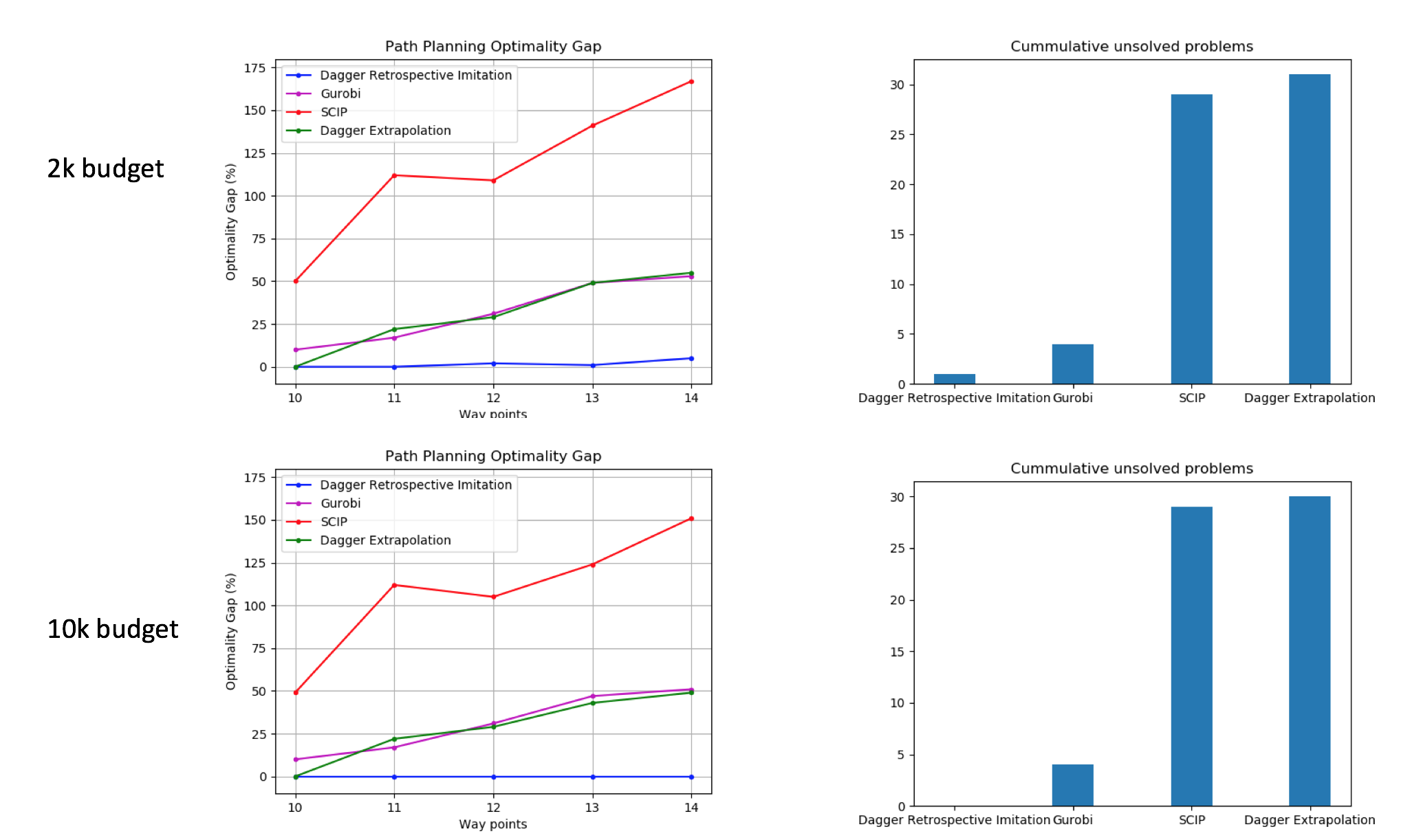}
\caption{Optimality gap comparisons and number of unsolved problem instances.}
\label{psulu_unsolved}
\end{figure}
Finally, we present statistics on how many instances of MILPs are not solved by each method when given a fixed budget on how many nodes to explore in the branch-and-bound tree. Retrospective DAgger achieves the best record among all the methods compared which implies that it is able to learn a stable and consistent solving policy.

\section{Experiments on Combinatorial Auction Test Suite}
\label{metric}

For completeness of comparison, we evaluate our approach on the same dataset as in \cite{he2014learning}, the Hybrid MILP dataset derived from combinatorial auction problems \citep{leyton2000towards}. For this experiment, we vary the number of bids, which is approximately the number of integer variables, from 500 to 730. Similar to \cite{he2014learning}, we set the number of goods for all problems to 100 and remove problems that are solved at the root. We use the select and pruner policy class to match the experiments in \cite{he2014learning} and a similar feature representation to those used in the path planning experiments.

The results of this experiment are shown in Figure \ref{fig:hybrid}. We see that neither retrospective imitation learning nor DAgger Extrapolation (\cite{he2014learning}) improves over SCIP. Upon further scrutiny of the dataset, we have found several issues with using this combinatorial auction dataset as a benchmark for learning search policies for MILPs and with the evaluation metric in \citep{he2014learning}.

Firstly, solvers like SCIP and Gurobi are well-tuned to this class of problems; a large proportion of problem instances is solved to near optimality close to the root of the branch-and-bound tree. As shown by Figure \ref{fig:our2017metric}, Gurobi and SCIP at the root node already achieve similar solution quality as SCIP and Gurobi node-limited by our policy's node counts; hence, exploring more nodes seems to result in little improvement. Thus the actual branch and bound policies matter little as they play a less important role for this class of problems.\\



\begin{figure}
\centering
    \begin{minipage}[b]{.45\textwidth}
        \centering
        \hbox{\hspace{-3ex}\includegraphics[width=1.2\textwidth]{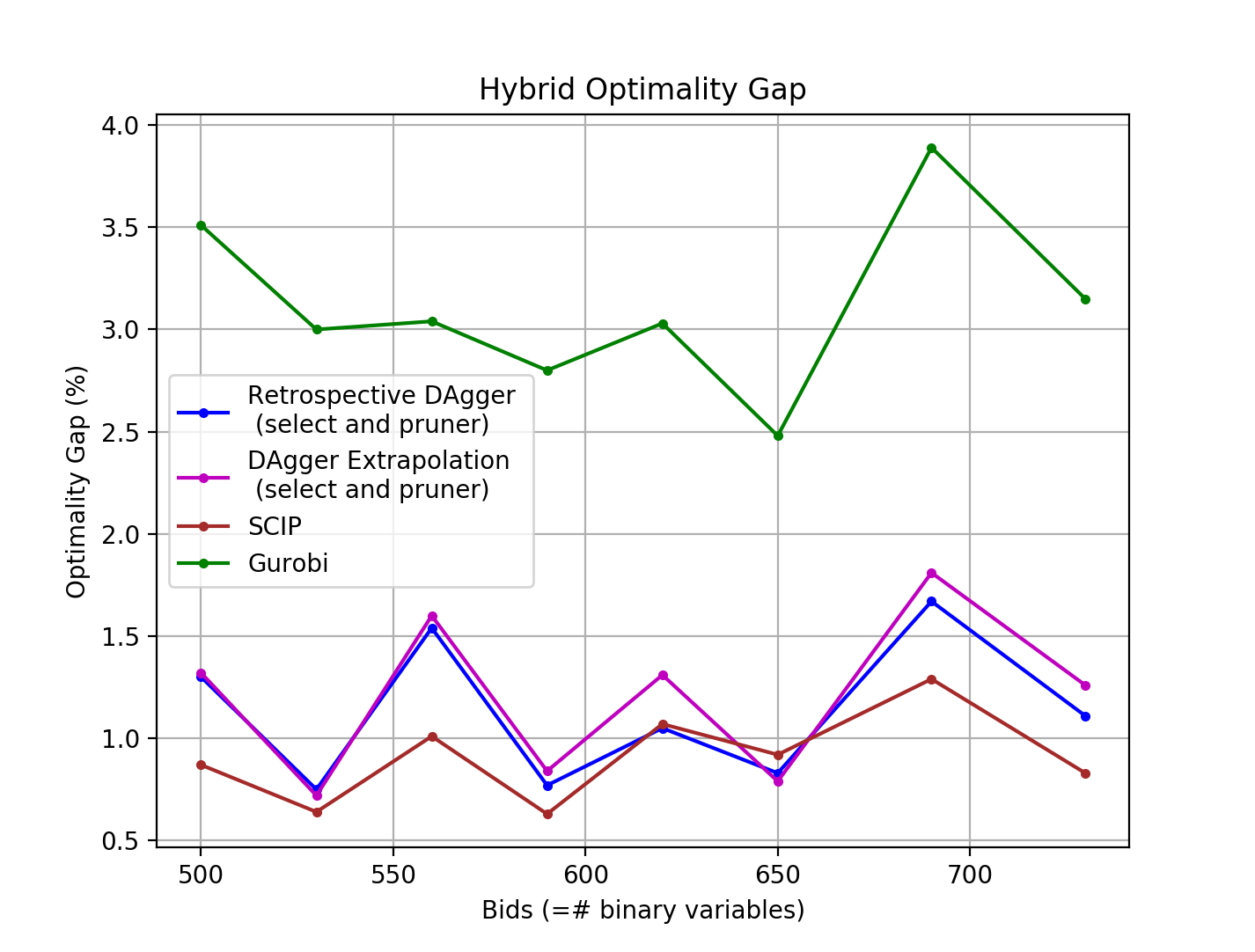}}    
        \caption{Comparison of optimality gap on the Hybrid combinatorial auction held-out test data.}
        \label{fig:hybrid}
    \end{minipage}%
\end{figure}
\begin{figure}
    \begin{minipage}[b]{.45\textwidth}
        \centering
        \hbox{\hspace{-4ex}\includegraphics[width=1.2\textwidth]{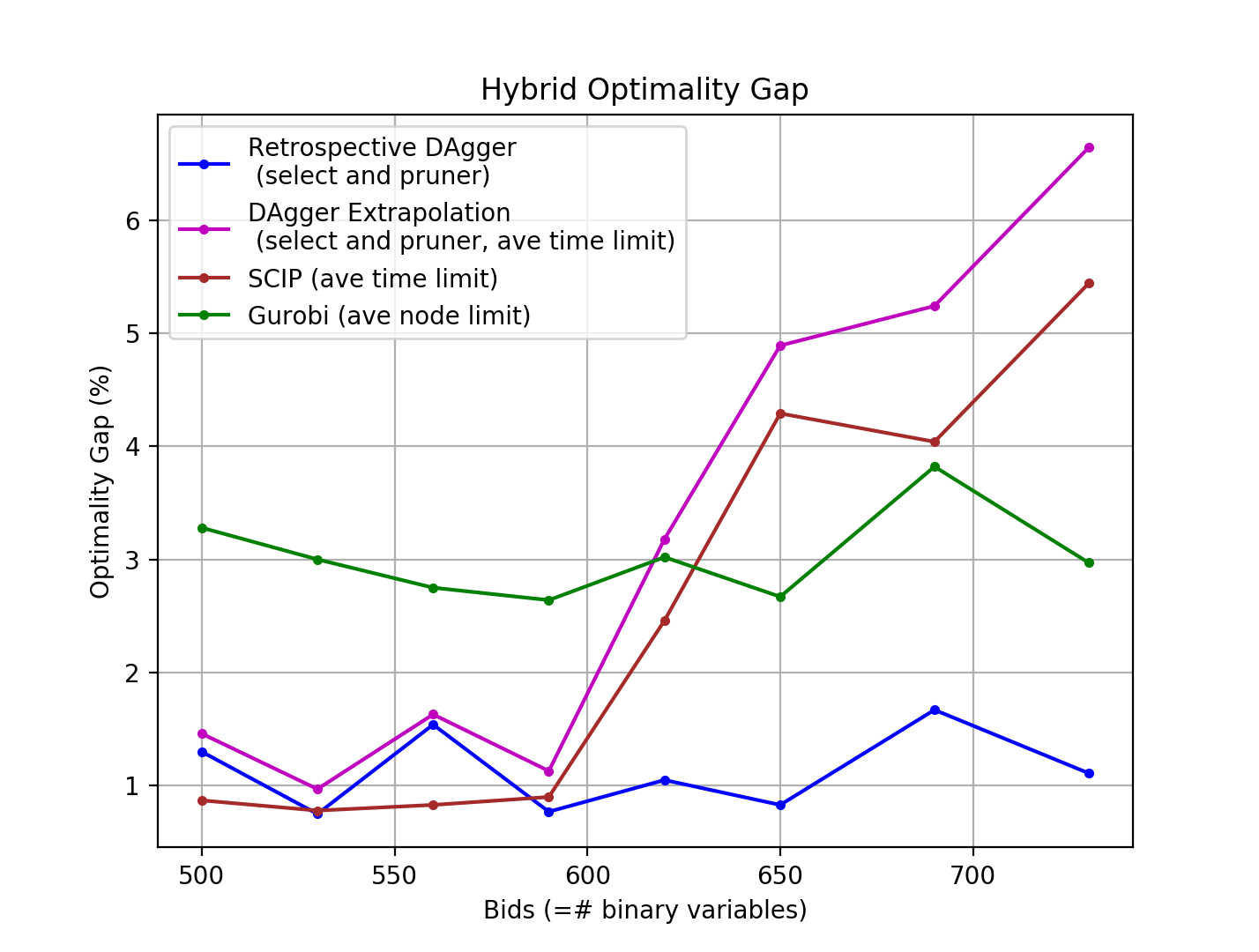}}    

        \caption{Comparison of optimality gap on the Hybrid combinatorial auction held-out test data using \citep{he2014learning} metric. \\}
        \label{fig:he2014metric}
	\end{minipage}%
    \vspace{.1in}
\end{figure}



\begin{figure}
\centering
    \begin{minipage}[b]{.45\textwidth}
        \centering
  		\hbox{\hspace{-3ex}\includegraphics[width=1.2\textwidth]{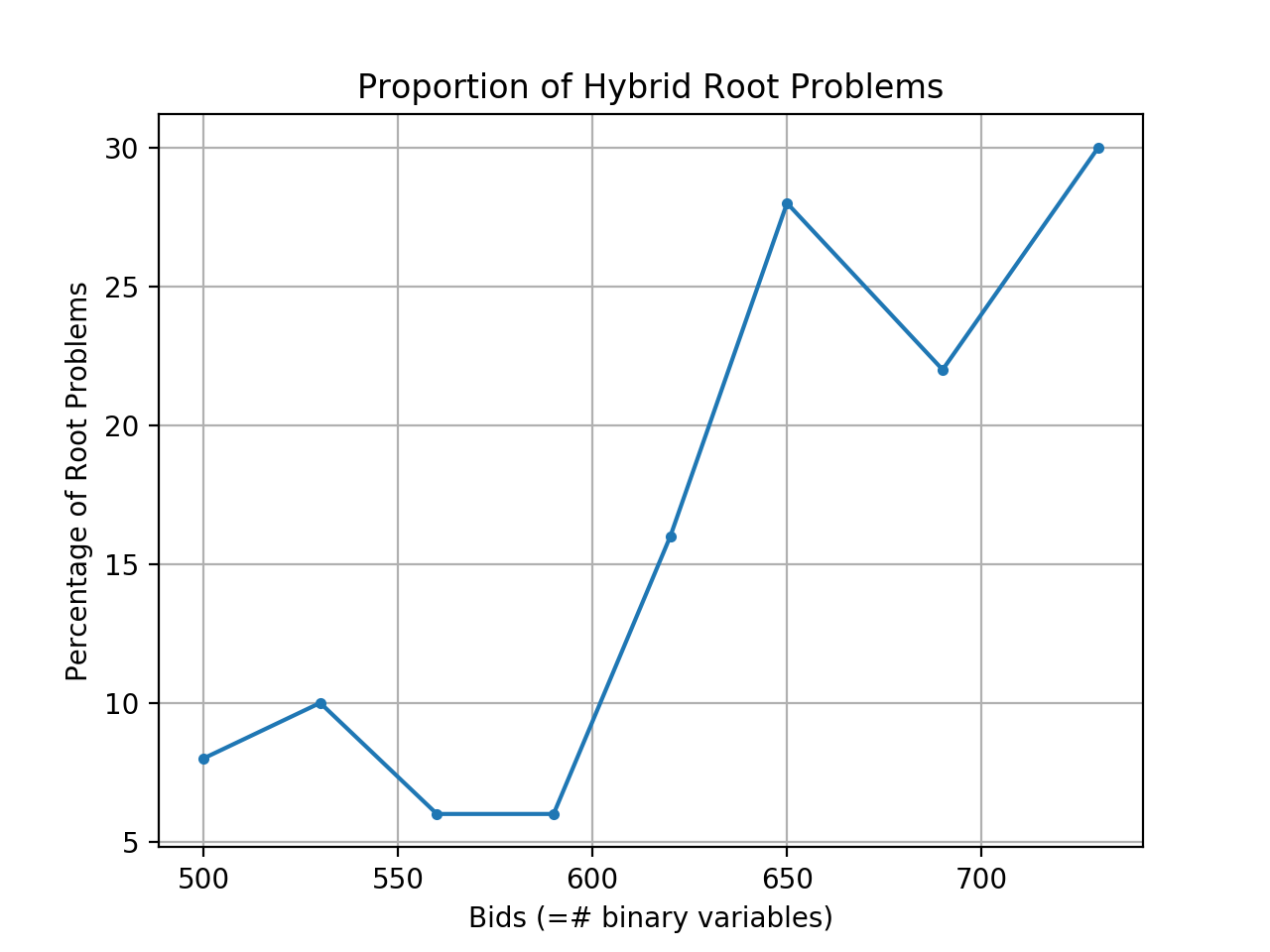}}
  		\caption{Proportion of Hybrid root problems. Root problems are problems for which SCIP limited by average runtime explores only the root node.}
  		\label{fig:rootprobprop}
	\end{minipage}%
\end{figure}
\begin{figure}
    \begin{minipage}[b]{.45\textwidth}
        \centering
        \hbox{\hspace{-4ex}\includegraphics[width=1.2\textwidth]{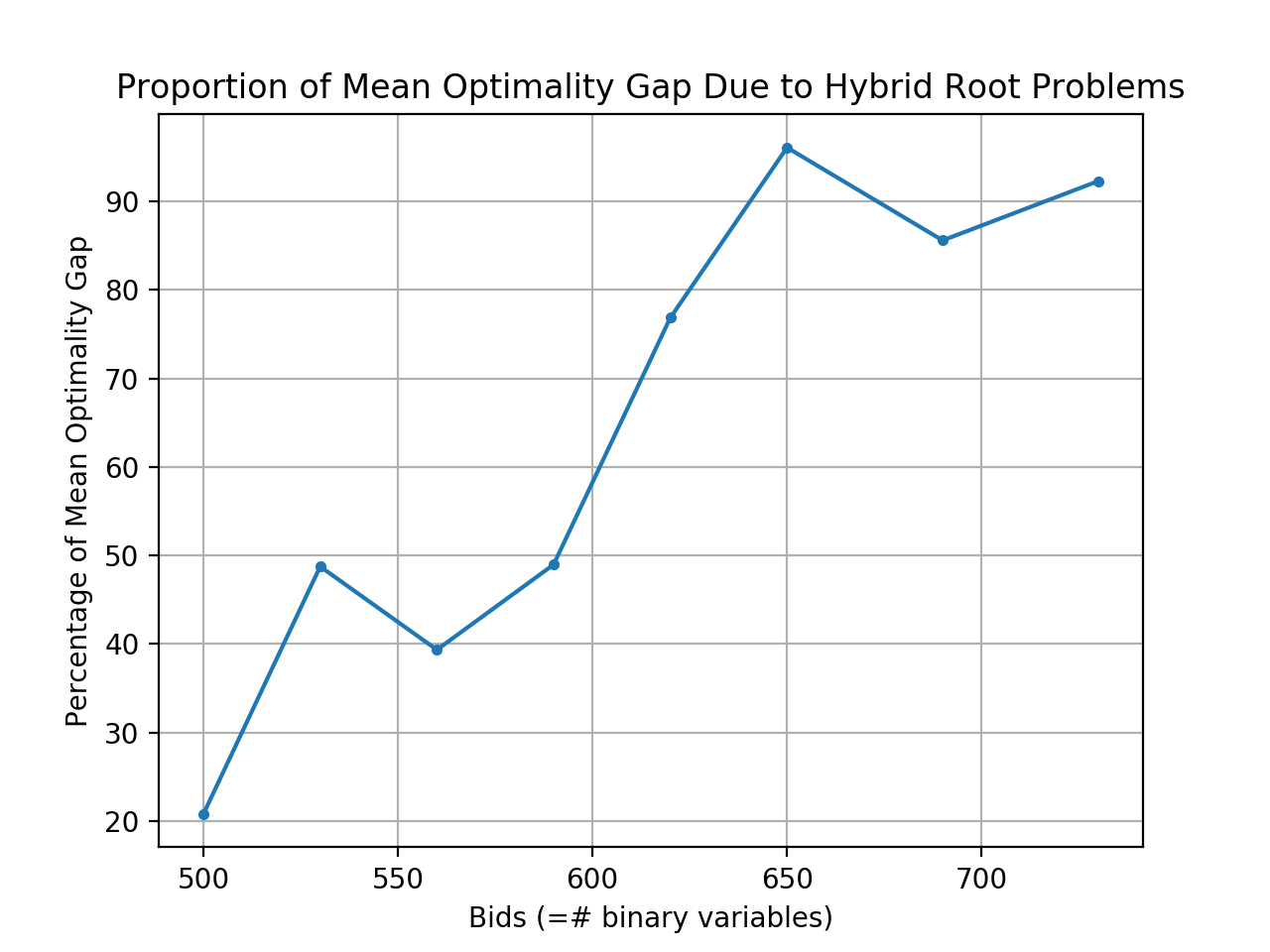}}    
        \caption{Proportion of mean optimality gap due to Hybrid root problems. \\ \\}
        \label{fig:rootprob}
    \end{minipage}%
    \vspace{.1in}
\end{figure}

Secondly, in this paper, we have chosen to use the number of nodes in a branch-and-bound search tree as a common measure of the speed for various solvers and policies. This is different from that used in \citep{he2014learning}, where the comparison with SCIP is done with respect to the average runtime. For completeness, we ran experiments using the metric in \citep{he2014learning} and we see in Figure \ref{fig:he2014metric} that retrospective imitation learning, upon scaling up, achieves higher solution quality than imitation learning and SCIP, both limited by the average runtime taken by the retrospective imitation policy, and Gurobi, limited by average node count. 

Instead of using average runtime, which could potentially hide the variance in the hardness across problem instances, using a different limit for each problem instance is a more realistic experiment setting. In particular, the average runtime limit could result in SCIP not being given sufficient runtime for harder problems, leading to SCIP exploring only the root node and a high optimality gap for these problems, which we call "root problems". As Figure \ref{fig:rootprobprop} shows, a significant proportion of the Hybrid held-out test set is root problems on larger scales. Figure \ref{fig:rootprob} shows that the majority of the mean optimality gap of SCIP limited by average runtime is due to the optimality gap on the root problems in the Hybrid dataset; for larger scale problems, this proportion exceeds 80\%, showing that limiting by average runtime heavily disadvantages SCIP.\\


\begin{figure}[t]
\centering
    \begin{minipage}[b]{.45\textwidth}
        \centering
  		\hbox{\hspace{-3ex}\includegraphics[width=1.2\textwidth]{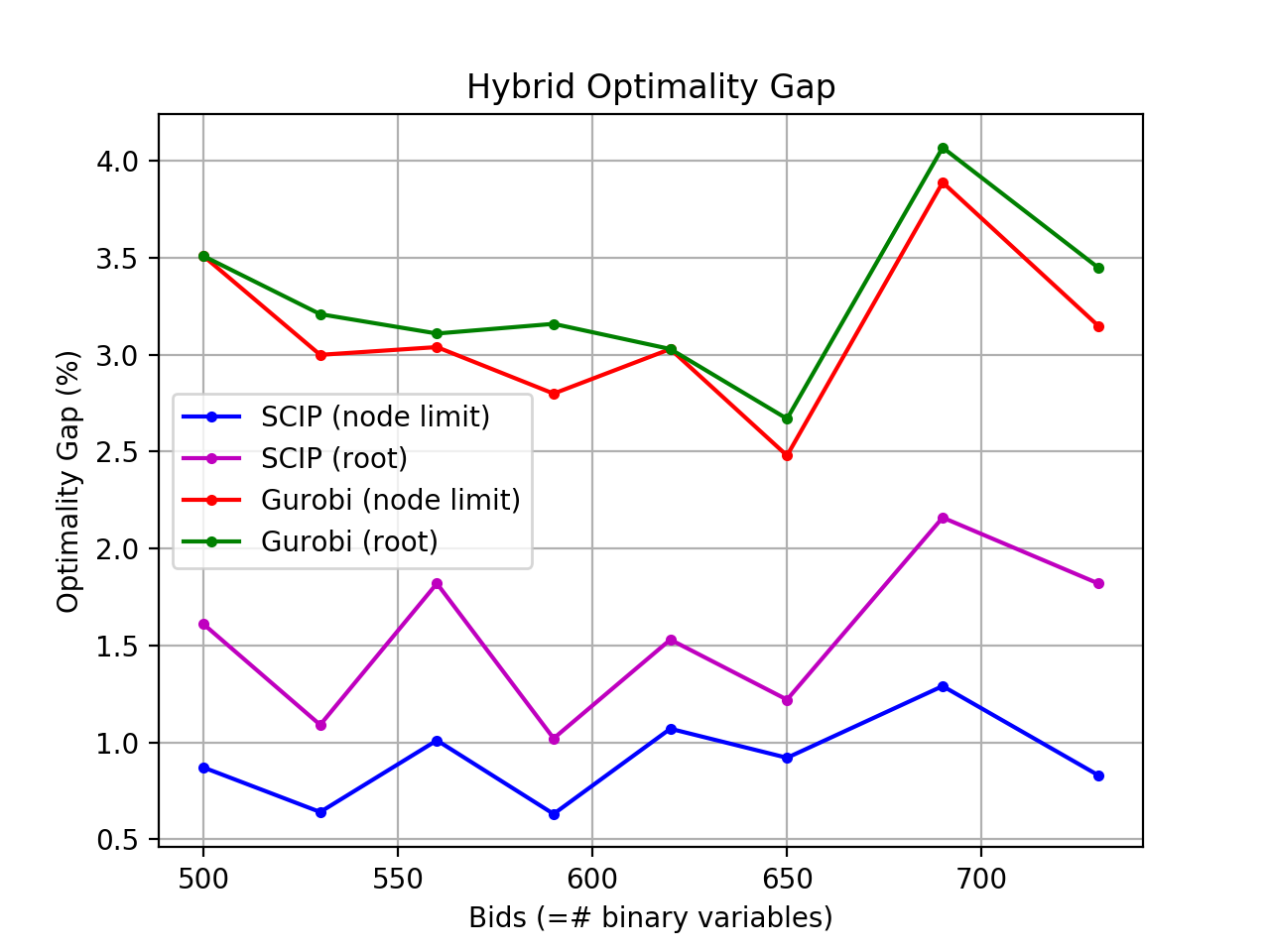}}
  		\caption{Comparison of optimality gap achieved by SCIP and Gurobi node-limited and at the root on the Hybrid combinatorial auction held-out test data.}
  		\label{fig:our2017metric}
	\end{minipage}%
\end{figure}
\begin{figure}
    \begin{minipage}[b]{.45\textwidth}
        \centering
  		\hbox{\hspace{-4ex}\includegraphics[width=1.2\textwidth]{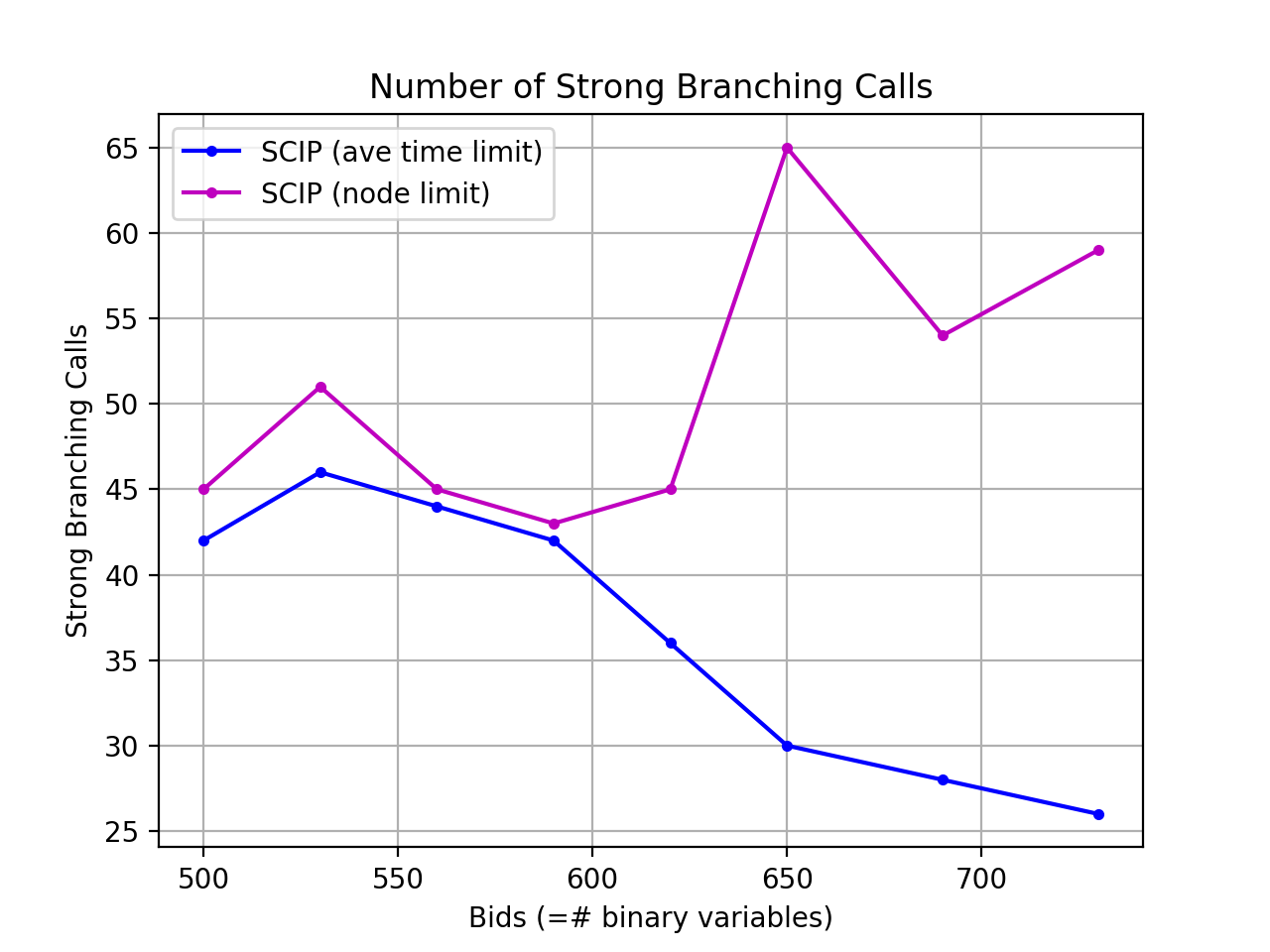}}
  		\caption{Number of strong branching calls at root for SCIP limited by average runtime and node count at every data point of retrospective imitation learning.}
  		\label{fig:sbranch_root}
	\end{minipage}%
    \vspace{.1in}
\end{figure}

Another issue is using runtime as the limiting criterion. From our observations, SCIP spends a substantial amount of time performing strong branching at the root to ensure good branching decisions early on. Limiting the runtime results in a limited amount of strong branching; as shown by Figure \ref{fig:sbranch_root}, SCIP limited by the average runtime of our retrospective imitation policies performs significantly less strong branching calls than SCIP limited by node counts, especially at larger problem sizes. In contrast, limiting the number of nodes does not limit the amount of strong branching since strong branching does not contribute to the number of nodes in the final branch-and-bound search tree. Considering the importance of strong branching for SCIP, we feel that only by allowing it can we obtain a fair comparison.\\
\\
As a result of the above reasons, we decided that the combinatorial auction dataset is not a good candidate for comparing machine learning methods on search heuristics and that the metric used in \citep{he2014learning} is not the best choice for validating the efficacy of their method.



\section{Exploration Strategy}
\label{sec:explore}
For retrospective imitation learning to succeed in scaling up to larger problem instances, it is important to enable exploration strategies in the search process. In our experiments, we have found the following two strategies to be most useful.
\begin{itemize}
\item $\epsilon$-greedy strategy allows a certain degree of random exploration. This helps learned policies to discover new terminal states and enables retrospective imitation learning to learn from a more diverse goal set. Discovering new terminal states is especially important when scaling up because the learned policies are trained for a smaller problem size; to counter the domain shift when scaling up, we add exploration to enable the learned policies to find better solutions for the new larger problem size.
\item Searching for multiple terminal states and choosing the best one as the learning target. This is an extension to the previous point since by comparing multiple terminal states, we can pick out the one that is best for the policy to target, thus improving the efficiency of learning.
\item  When scaling up, for the first training pass on each problem scale, we collect multiple traces on each data point by injecting 0.05 variance Gaussian noise into the regression model within the policy class, before choosing the best feasible solution. 
\end{itemize}

\end{document}